\documentclass[final,3p,times]{elsarticle}

\usepackage{graphics} 
\usepackage{algorithm}
\usepackage{algpseudocode}
\usepackage{amssymb}
\usepackage{amsmath}
\usepackage{amsthm}
\usepackage{booktabs}
\usepackage{mathtools}
\usepackage{multicol}
\usepackage{multirow}
\usepackage[numbers]{natbib}

\usepackage{hyperref}
\usepackage{subfigure}
\usepackage{wrapfig}
\newtheorem{theorem}{Theorem}
\newtheorem{definition}{Definition}
\newtheorem{problem}{Problem}
\newtheorem{lemma}{Lemma}

\usepackage{color}

\usepackage{soul}

\hypersetup{
    colorlinks=true,      
    linkcolor=SkyBlue,       
    citecolor=green,      
    urlcolor=cyan         
}
\journal{Transportation Research: Part C}
\allowdisplaybreaks

\begin{document}

\begin{frontmatter}


\title{Optimization of Multi-Agent Flying Sidekick Traveling Salesman Problem over Road Networks\tnoteref{label1}}
\author[mit]{Ruixiao Yang\corref{cor1}}
\ead{ruixiao@mit.edu}
\author[mit]{Chuchu Fan\corref{cor2}}
\ead{chuchu@mit.edu}
\cortext[cor1]{Corresponding author.}
\affiliation[mit]{organization={Department of Aeronautics and Astronautics, Massachusetts Institute of Technology},
            city={Cambridge},
            postcode={02139},
            state={MA},
            country={USA}}





\begin{abstract}
The mixed truck-drone delivery systems have attracted increasing attention for last-mile logistics, but real-world complexities demand a shift from single-agent, fully connected graph models to multi-agent systems operating on actual road networks. We introduce the multi-agent flying sidekick traveling salesman problem (MA-FSTSP) on road networks, extending the single truck-drone model to multiple trucks, each carrying multiple drones while considering full road networks for truck restrictions and flexible drone routes. We propose a mixed-integer linear programming model and an efficient three-phase heuristic algorithm for this NP-hard problem. Our approach decomposes MA-FSTSP into manageable subproblems of one truck with multiple drones. Then, it computes the routes for trucks without drones in subproblems, which are used in the final phase as heuristics to help optimize drone and truck routes simultaneously. Extensive numerical experiments on Manhattan and Boston road networks demonstrate our algorithm's superior effectiveness and efficiency, significantly outperforming both column generation and variable neighborhood search baselines in solution quality and computation time. Notably, our approach scales to more than 300 customers within a 5-minute time limit, showcasing its potential for large-scale, real-world logistics applications.
\end{abstract}




\end{frontmatter}



\section{Introduction}
\label{sec:introduction}
In the rapidly evolving landscape of aerial robotics, specifically drones, the reduction they can contribute to road traffic through aerial package delivery has drawn much attention. However, due to their limited battery capacity and carrying ability, drones are usually considered to accomplish tasks with the help of a ground vehicle. A common scenario is a delivery system where trucks carry drones to send packages to customers cooperatively. The Flying Sidekick Traveling Salesman Problem (FSTSP)~\cite{murray2015flying}, which is a variant of the famous TSP problem, models the problem of a drone working in tandem with a delivery truck to visit every customer. The problem captures the difficulty of synchronization between the truck and the drone but oversimplifies the path-finding problem between customers for the truck-drone system by an edge in the graph. 

\begin{figure}[t!]
    \centering
    \subfigure[Phase 1: assign customers to truck groups]{
    \includegraphics[width=0.31\linewidth]{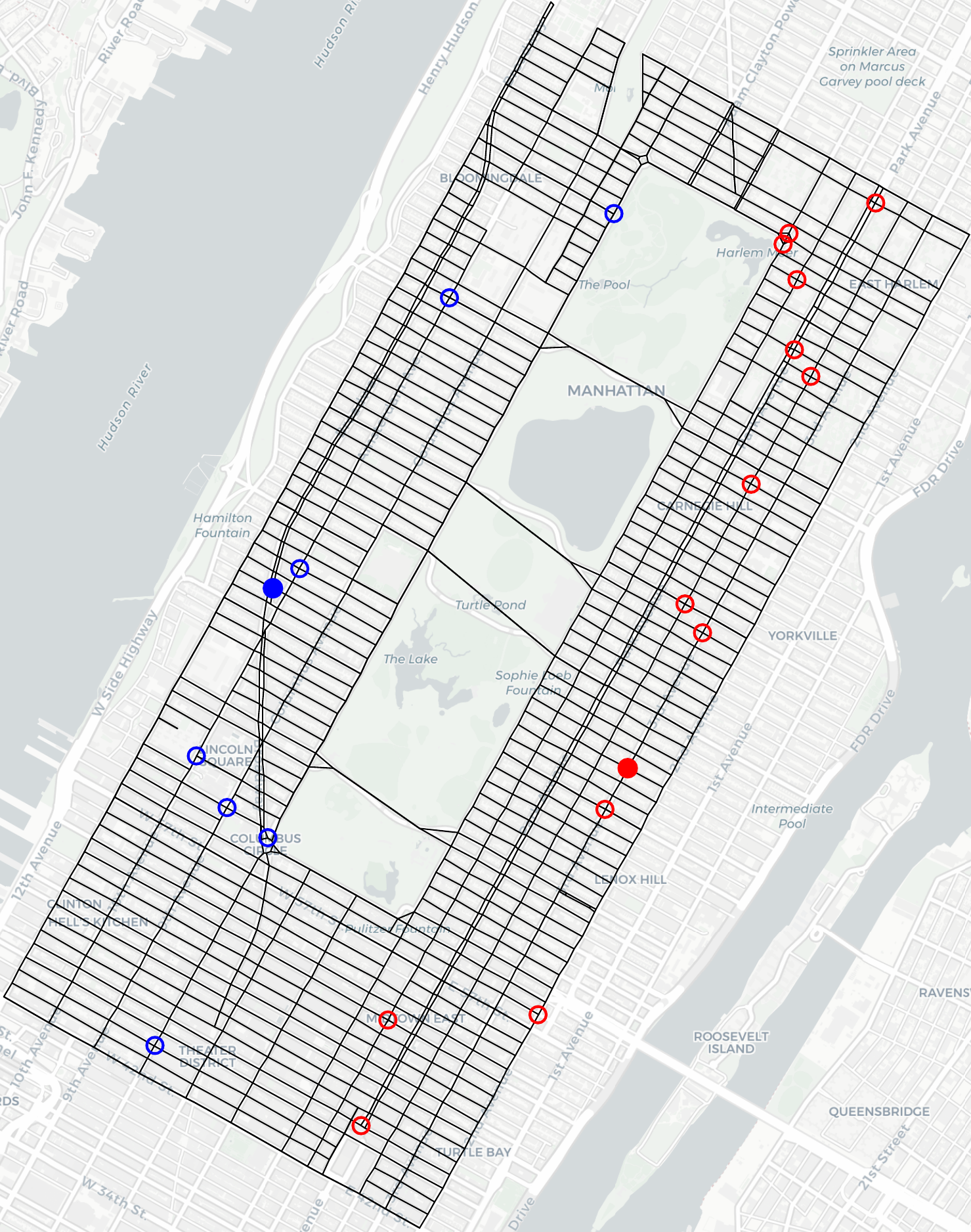}
    \label{fig: alpha manhattan}
    }
    \subfigure[Phase 2: solve the Set-TSP to get visiting orders of customers]{ 
    \includegraphics[width=0.31\linewidth]{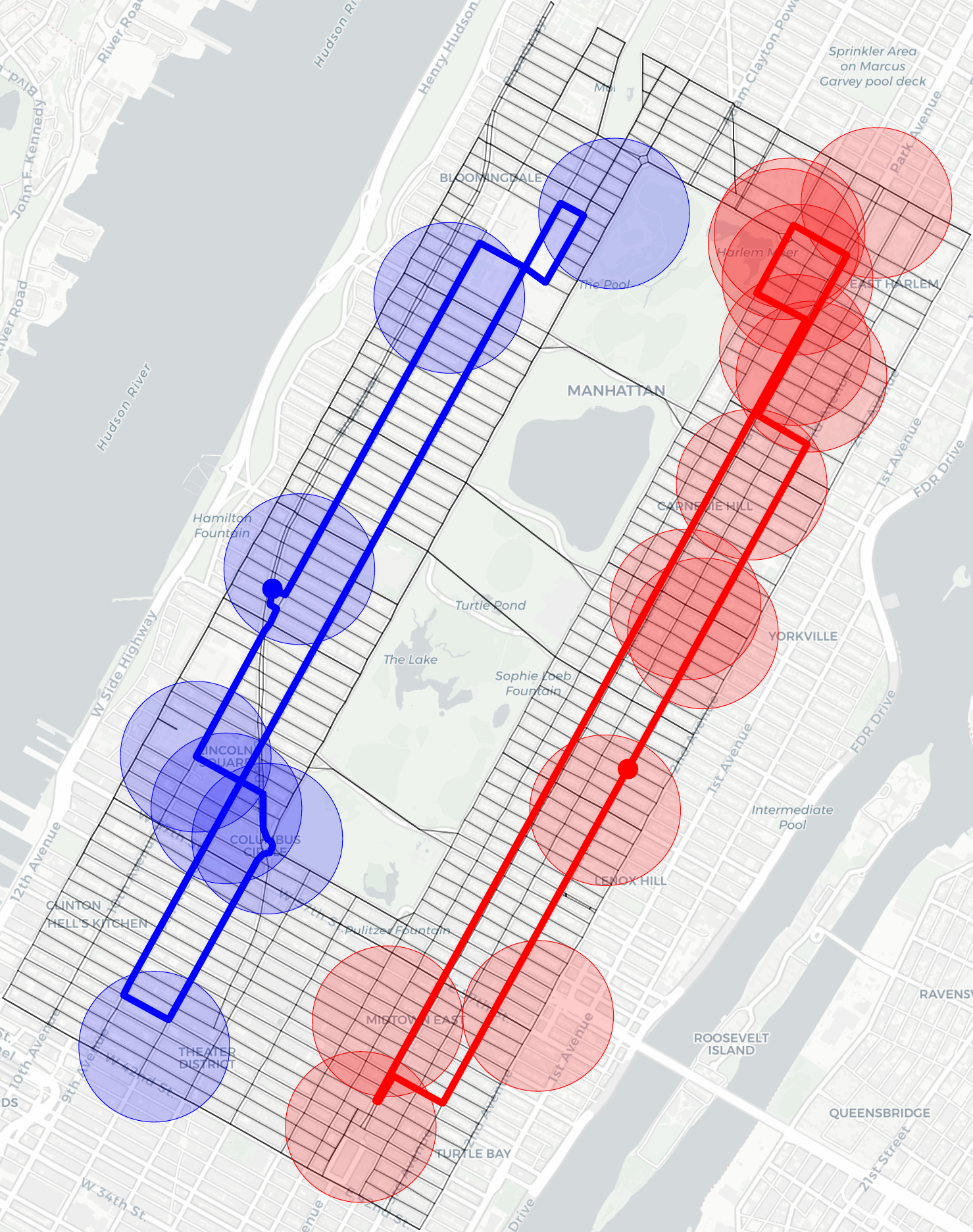}
    \label{fig: alpha boston}
    }
    \subfigure[Phase 3: optimize routes for trucks and drones simultaneously based on visiting orders]{
    \includegraphics[width=0.31\linewidth]{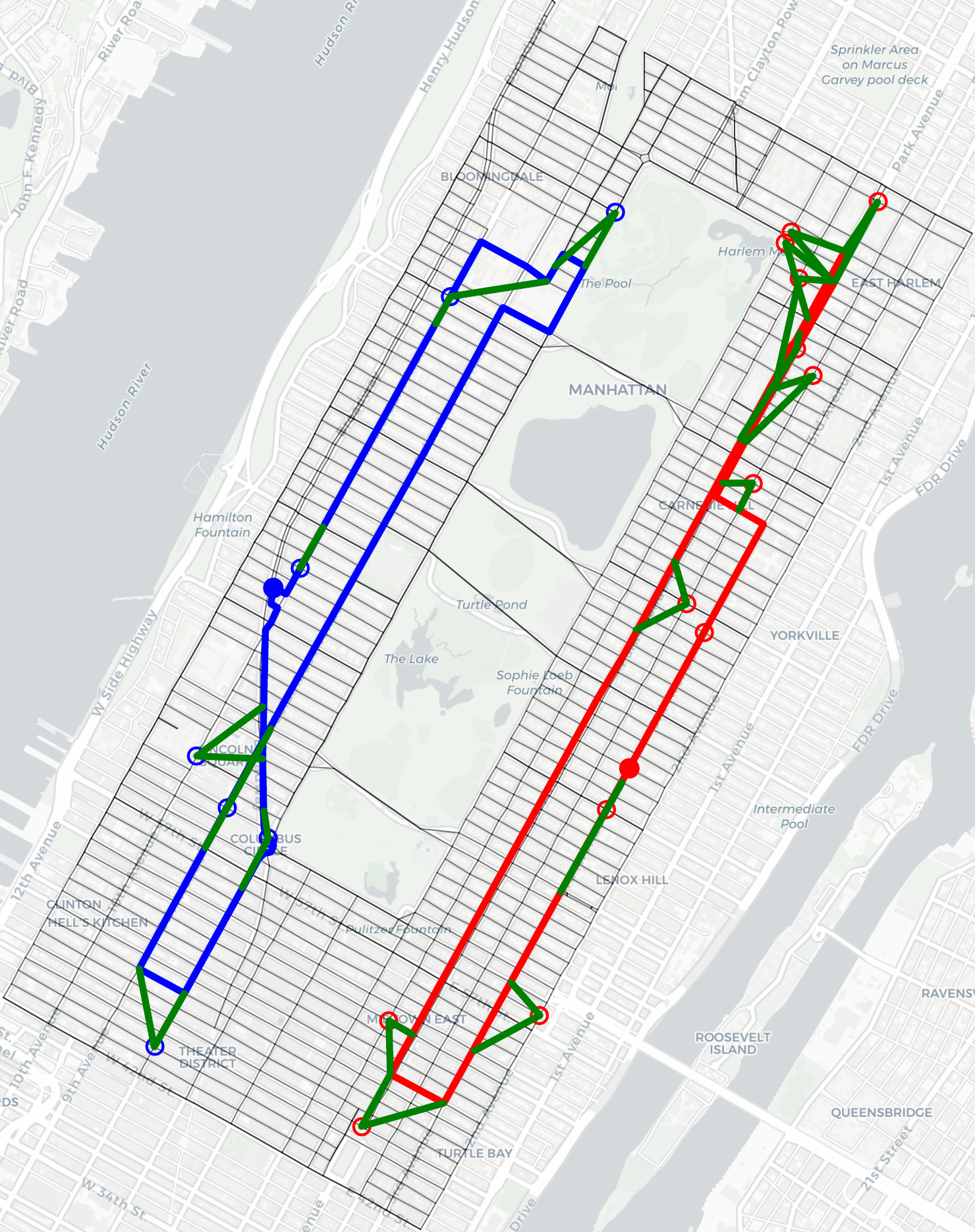}
    }
    \caption{An Example of the output of each phase in the algorithm. In phase 1, we assign customers, marked by hollow circles, to their closest depots, marked by solid circles, as shown in Fig. (a). In phase 2, we collect nodes within half of the distance limit for drones as the circles shown in Fig. (b) and then apply the Set TSP method to find the shortest tour to visit all sets for each truck group starting from its depot as the colored lines shown in Fig. (b). In phase 3, based on the visiting order on routes found in phase 2, we optimize the route for truck and drone together. The final results are shown in Fig. (c), where the red and blue lines are routes for trucks of different groups, and the green lines are routes for drones when they are separated from the truck to visit customers.}
    \label{fig:3 phase overview}
\end{figure}

In this paper, to better capture the truck-drone delivery problem in the real world, we consider a multi-agent version of FSTSP over a road network, termed Multi-Agent FSTSP (MA-FSTSP). The FSTSP considers the TSP problem of one truck loaded with one drone on a fully connected graph, which only consists of one node to depart and to return called \textit{depot} and a set of nodes to visit called \textit{customers}. The \textbf{two differences} in MA-FSTSP setting from FSTSP are: (1) we consider multiple trucks carrying multiple drones instead of one truck with a single drone; (2) we consider the problem on real-world road maps which contain (many) nodes other than depots and customers and hence the graph is no longer fully connected. Those nodes can represent any location or address on the map. In MA-FSTSP, multiple truck groups start from different depots to visit a given set of customers. Each truck group consists of one truck and a fleet of drones, in which the drones are loaded on the truck. Similar to TSP, in MA-FSTSP, each customer must be visited exactly once by one member (a truck or a drone) in one of the truck groups. In addition, the trucks and drones must synchronize at the same node on the graph for the takeoff and land actions. While the trucks are constrained to move on the roads in the road network, the drones can fly freely between any pair of nodes. To model the limited battery capacity and carrying ability in practice, each drone is restricted to a maximum travel distance and a maximum number of customers it can serve before returning to the truck after each separation. Additionally, we only allow drones to land on the same truck from which they take off. Given this setup, the MA-FSTSP asks for the route that minimizes the total time for all truck-drone groups to finish the delivery job (i.e., all customers are visited.) It is worth noting that MA-FSTSP is also NP-hard as the TSP can be reduced to the MA-FSTSP by setting the number of truck groups to 1 and the number of drones in each group to 0.

Literature on multi-agent FSTSP and cooperative Multi-Agent Path Finding (MAPF) on road networks is scarce. To the best of our knowledge, only the methods in \citet{lin2022discrete} and \citet{gao2023scheduling} are designed for FSTSP variants on graphs of road networks and solved them through variants of the Variable Neighborhood Search (VNS) method. The MAPF community, on the other hand, has not yet explored the collaboration between different agents well. Existing works~\cite{choudhury2021efficient, greshler2021cooperative, choudhury2021coordinated} adapted the Conflict-Based Search (CBS) method for the cooperation. However, this search-based method suffers from scalability issues unless a decoupling mechanism is applied to divide the problem into smaller subproblems to solve it state-wise.

In this paper, we propose a novel 3-phase method (see Figure \ref{fig:3 phase overview}) for solving MA-FSTSP that can achieve state-of-the-art performance and good scalability. In the first phase, customers are assigned to the truck group starting from their closest depots under a set-based distance metric. In the second phase, we extend a TSP to a Set TSP by grouping the road nodes within a given distance to the customers to form a set. The Set TSP asks for a route that visits each set exactly once instead of the requirement of visiting each customer exactly once, which is designed to approximate the usage of drones. We adopt the method from \citet{marcucci2021shortest} to solve the Set TSP via a Mixed-Integer Linear Programming (MILP). 
Finally, we plan for trucks and drones simultaneously to find the optimal routes that visit the customers in the same order as the Set TSP routes found in phase 2.

We evaluate our approach using real-world road networks from Manhattan and Boston, comparing it against established baselines from recent literature, including the work of~\citet{gao2023scheduling} and a Hill Climbing combined Variable Neighborhood Search (HC-VNS) method. Our results demonstrate significant improvements in both solution quality and computational efficiency, including a more than 30\% cost reduction to the baseline methods when the problem size is large and a successful scaling up to solve problems involving more than 300 customers. The efficiency of our approach stems from the polynomial complexity of the first and third phases and the quadratic number of binary variables relative to the map and customer sizes in the second phase. We validate the set-based methods by comparing the performance with their node version on Manhattan and observe cost savings of up to $11.84\%$. Through sensitivity analysis, we also derive insights into the factors influencing the performance of the cooperative delivery systems.


The rest of the paper is organized as follows. \autoref{sec:literature} reviews prior works related to multi-agent extensions of FSTSP and path-finding problems on real-world road networks. \autoref{sec:problem} introduces the notations and formally defines the MA-FSTSP. We then describe the MILP model and our method in detail in \autoref{sec:milp} and \autoref{sec:methodology}. Numerical experiments are presented in \autoref{sec:experiments}. We conclude our work and point out potential directions in \autoref{sec:conclusion}.

\section{Related work}
\label{sec:literature}
In practice, especially in drone-assisted delivery tasks, trucks usually start from and end at the same depot, framing the problem as a Traveling Salesman Problem (TSP) or a Vehicle Routing Problem (VRP). This context is crucial as it reflects real-world logistical constraints and optimization needs. The truck-drone delivery system in the delivery problem was first studied by \citet{murray2015flying}, who introduced the Flying Sidekick Traveling Salesman Problem (FSTSP). Their work laid the foundation for optimizing combined truck and drone deliveries, sparking numerous variants of this problem.

One significant research direction has been the generalization of the FS-TSP to multi-agent versions. These multi-agent variants can be classified into three main categories: (1) one truck carrying multiple drones, which expands the drone capacity of a single truck, potentially increasing delivery efficiency (\citet{mbiadou2018iterative, karak2019hybrid, murray2020multiple, cavani2021exact, salama2022collaborative, bruni2022logic}); (2) multiple trucks with one drone each, which focuses on the task allocation among truck-drone systems to consider a wider geographical coverage (\citet{sacramento2019adaptive, chiang2019impact, lu2022multi, luo2022last}); (3) multiple trucks with multiple drones each, which optimizes the previous two models simultaneously (\citet{kitjacharoenchai2019multiple, tamke2021branch, chen2021adaptive, gao2023multi}). Each variant builds upon the original FSTSP model, addressing different practical scenarios and optimization challenges in truck-drone delivery systems. Given the computational complexity of solving these multi-agent FSTSP variants exactly using mixed-integer linear programming (MILP), researchers have developed many heuristics and metaheuristics to approximate optimal solutions efficiently. For the single-truck, multi-drone variant, approaches include the Clark and Wright heuristic \citet{karak2019hybrid}, MILP models combined with three-phased heuristics \citet{murray2020multiple}, and a two-phase method integrating simulated annealing (SA) and variable neighborhood search (VNS) \citet{salama2022collaborative}. In the multi-truck, single-drone category, methods range from adaptive large neighborhood search \citet{sacramento2019adaptive} and genetic algorithms \citet{chiang2019impact} to evolutionary algorithms with specialized local search operators \citet{lu2022multi} and iterated local search \citet{luo2022last}. For the most complex variant with multiple trucks and multiple drones, researchers have employed adaptive insertion algorithms \citet{kitjacharoenchai2019multiple}, adaptive large neighborhood search heuristics \citet{chen2021adaptive}, and column generation methods \citet{gao2023multi}. This diversity of approaches underscores both the challenge of solving these problems and the ongoing efforts to develop efficient approximation methods that balance solution quality with computational tractability.

While the multi-agent variants significantly expanded the scope of the FSTSP, most early models shared a common limitation in their representation of the operational environment. These models typically consider a fully connected graph consisting of only customers and depots. However, this simplification does not adequately model the complexity of real-world environments, particularly road networks, which can significantly impact routing decisions and overall system efficiency. 
Recognizing this limitation, \citet{carlsson2018coordinated} made a significant advancement by proposing that drones should be able to take off from any node in the 2D plain, not just depots and customer nodes. Such an extension enlarges the solution space too much, as the truck cannot move freely in practice. \citet{li2022truck} restricted the flexibility by allowing drones to synchronize only with trucks on arcs between customers rather than arbitrary locations. Both works must deal with the nonlinear term introduced with the continuous synchronization location space and suffer from the high computational cost. Recent works present a more efficient model. \citet{lin2022discrete} and \citet{gao2023scheduling} extended the problem to incorporate actual road networks, allowing drones to take off and land at arbitrary nodes within the network. This approach provides a much more realistic model of the operational environment, potentially leading to more practical and efficient delivery strategies. Both works still exhibit poor scalability, limiting their applicability to larger, real-world instances. 

One relevant research problem, the Multi-Agent Path Finding (MAPF), also studies the synchronization challenge as in FSTSP and its variants. For instance, \citet{greshler2021cooperative} considered the pairwise collaboration constraint, which requires two agents to meet at a specific time and place. In the context of drone deliveries, \citet{choudhury2021efficient} optimized drones that utilize public transit to save energy, while \citet{choudhury2021coordinated} simultaneously optimized truck and drone routes in a coordinated system without the TSP constraints. These developments in MAPF provide valuable insights and potential solution approaches for the complex routing and coordination challenges in multi-agent truck-drone delivery problems.

\section{Problem description}
\label{sec:problem}
\begin{figure}
    \centering
    \subfigure[Truck Route (red)]{
    \includegraphics[width=0.46\linewidth]{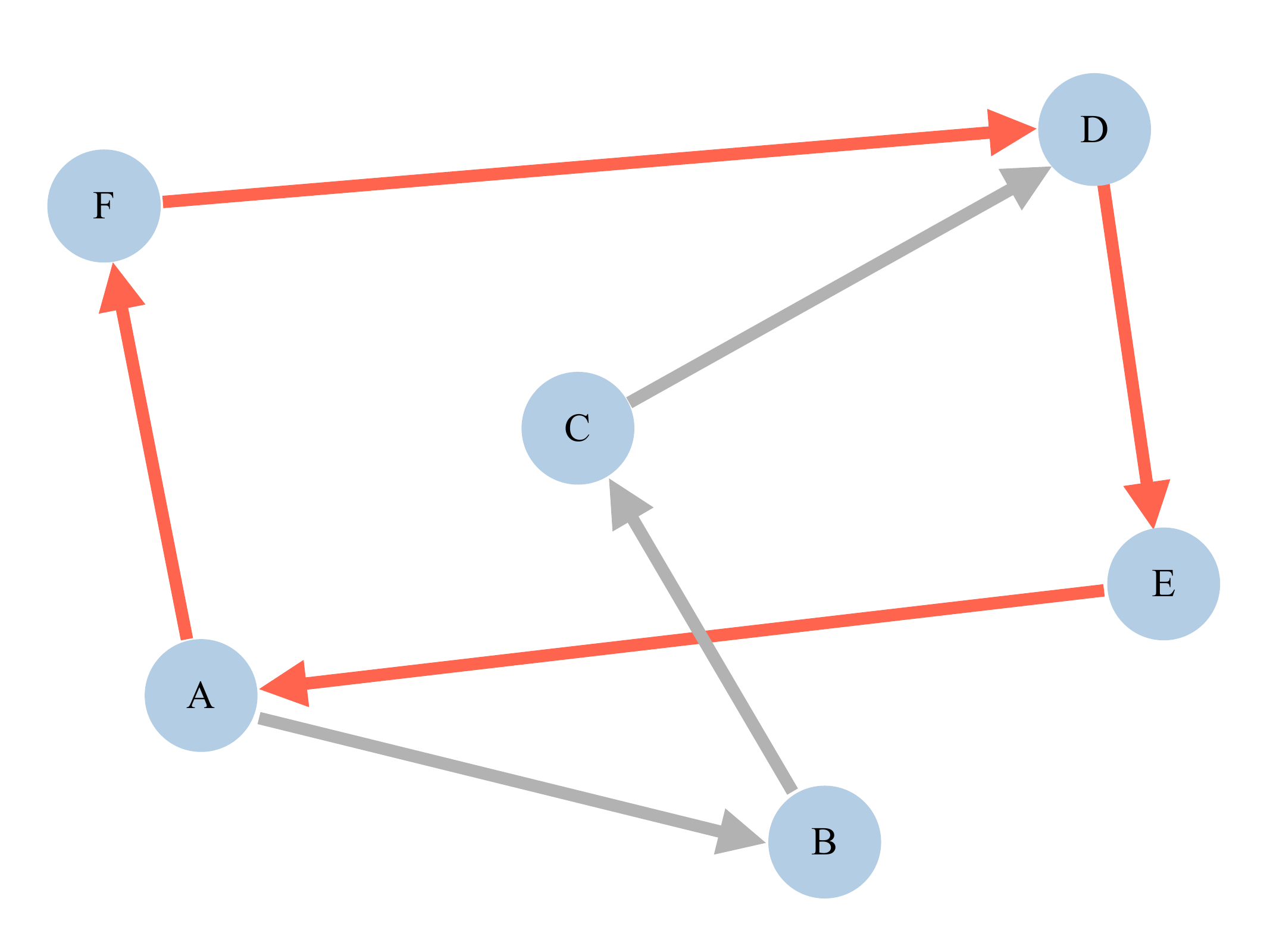}
    \label{fig: truck}
    }
    \subfigure[Drone Route (green)]{
    \includegraphics[width=0.46\linewidth]{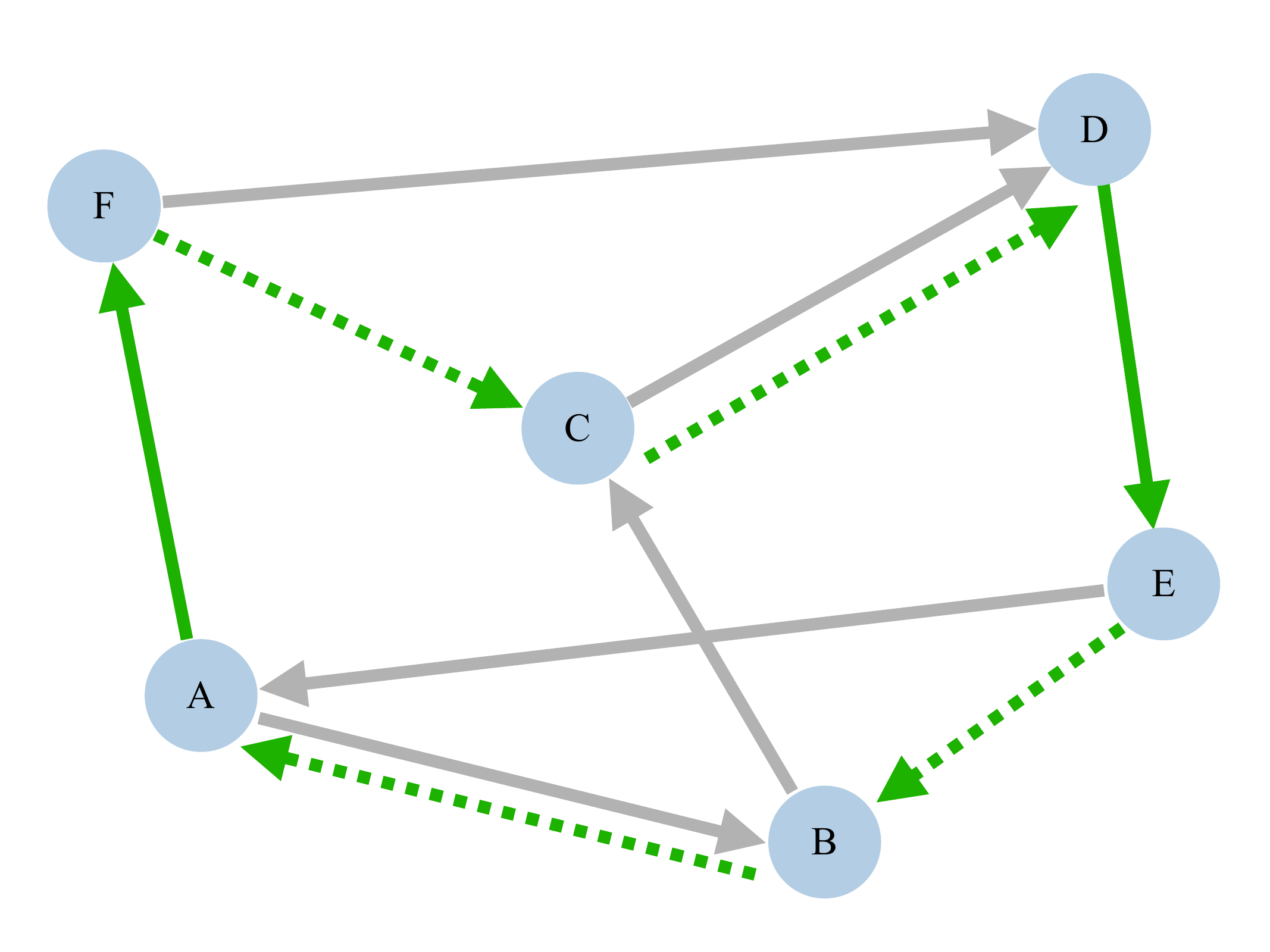}
    \label{fig: drone}
    }
    \caption{An example of the truck and drone routes on a directed graph with node 1$\sim$6. The solid arrows (in all colors) represent the road that the truck must follow. The truck route (A, F, D, E, A) is colored red and the drone route (A, F, C, D, E, B, A), represented as $\left<(\text{A, F, D, E, A}), \{(2,\text{C},3), (5,\text{B},6)\}\right>$, where 2, 3, 4, 5 are indices of vertices F, D, E, A in the truck route, is colored green. The solid green arrows mean the drone is carried by a truck, and the dashed green arrows mean the drone is flying freely.}
    \label{fig: truck and drone route}
\end{figure}
The MA-FSTSP requires finding the fastest tours for each truck group to visit a set of customers, ensuring each customer is visited exactly once while meeting the following constraints: (1) trucks can only travel on roads, whereas drones can fly directly between locations after departing from the truck; (2) drones have a limited flight distance due to battery capacity and must return to their truck; (3) drones can be dispatched from and collected by the same truck at different locations; (4) takeoff and landing actions occur only at road endpoints for simplicity; (5) drones can visit only one customer before returning to the truck due to limited payload capacity.

Formally, MA-FSTSP considers shared road network represented as a strongly connected directed graph\footnote{A directed graph is called strongly connected if there is a path in each direction between each pair of graph vertices.} $\mathcal{G}=(\mathcal{V}, \mathcal{E})$, where $\mathcal{V}$ is the set of vertices and $\mathcal{E}$ is the set of edges. Each pair of vertices $u, v\in\mathcal{V}$ is associated with a distance $d(u,v)\ge 0$, which may represent the Euclidean distance on a planar map or geometric distance on Earth. Let $\mathcal{P}=\{p_1,p_2,\cdots,p_m\}\subseteq \mathcal{V}$ be the set of $m$ depots for each truck group to start and end their tours, and $\mathcal{C}=\{c_1,c_2,\cdots,c_n\}\subseteq \mathcal{V}$ be the set of $n$ customers to visit. A customer $c\in \mathcal{C}$ can be visited by any agent starting from any depot $p\in\mathcal{P}$. Furthermore, we assume $\mathcal{C}\cap\mathcal{P}=\varnothing$.

We plan for a fleet of agents comprising $m$ truck groups starting from different depots, denoted as $\mathcal{T}_p$ for each $p\in\mathcal{P}$. Each truck group consists of a truck carrying $k$ drones, denoted as $\mathcal{D}_p$. We assume that trucks move along the edges at a constant speed $s_{\text{tr}}>0$ and drones fly at a constant speed $s_{\text{dr}}>0$ after being dispatched from the truck. Drones can take off from and land on their truck at any $v\in\mathcal{V}$ and can fly directly between any pair of vertices, independent of the road network. Once dispatched, the drone is limited to serve one customer due to the loading capacity and is limited to fly at most $r$ due to the battery capacity before landing back on the truck. 


Next, we formally define the \textit{truck route} and the \textit{drone route} that constitute a solution to the MA-FSTSP.
\begin{definition}[\textbf{Truck Route}]
    We call a sequence of vertices $(v_0, v_1,\cdots, v_l)$, $v_j\in\mathcal{V},\forall j\in \{0,1,\cdots,l\}$, a truck route if
    $v_0=v_l=p\in\mathcal{P}$, and $(v_j,v_{j + 1})\in\mathcal{E}$, $\forall j\in\{0,1,\cdots, l-1\}$. 
\end{definition}
Intuitively, a truck route forms a cycle in the graph, starting and ending at a specific depot, as illustrated in \autoref{fig: truck}. Before discussing drone routes, we define a drone delivery tuple to represent these routes effectively.
\begin{definition}[\textbf{Drone Delivery Tuple}]
    Given a truck route $(v_0, v_1,\cdots, v_l)$, a drone delivery tuple $(x, c, y)$ describes a delivery task that a drone carried by the truck takes off from $v_x$, visits customer $c$, and lands on the truck at $v_y$. 
\end{definition}
\begin{definition}[\textbf{Drone Route}]
    A drone route is represented by a tuple $\left<(v_0, v_1,\cdots, v_l), \{(x_1, c_1, y_1),(x_2,c_2,y_2),\cdots\}\right>$, where $(v_0, v_1,\cdots, v_l)$ is the truck route that the drone accompanies, and $\{(x_1, c_1, y_1),(x_2,c_2,y_2),\cdots\}$ comprises drone delivery tuples of the delivery tasks executed by drones. 
\end{definition}
\autoref{fig: drone} shows an example of the drone route. We define the truck group route via the truck route and drone routes. 
\begin{definition}[\textbf{Truck Group Route}]
    Given truck group $\mathcal T_p$ starting from depot $p$, the route is represented as a $(k+1)$-tuple $\pi_p$, where $\pi_p[0]=(v_0, v_1, \cdots, v_l)$ is a truck route that $v_0=v_l=p_i$, $\pi_p[j]=\{(x_1^{(j)}, c_1^{(j)}, y_1^{(j)}),(x_2^{(j)},c_2^{(j)},y_2^{(j)}),\cdots\}$, $j=1,2,\cdots k$, are $k$ sets of drone delivery tuples such that $\left<\pi_p[0], \pi_p[j]\right>$, $j=1,2,\cdots k$, are all drone routes. 
\end{definition}
The total cost of a truck group route $\pi_p$, denoted as $\text{cost}(\pi_p)$, is defined as the minimum time needed for the truck group $\mathcal T_p$ to address all drone delivery tasks and return to the depot $p$. Now, we formally state the MA-FSTSP.
\begin{problem}[\textbf{MA-FSTSP}]
Given a strongly connected directed graph $\mathcal{G}=(\mathcal{V}, \mathcal{E})$, and $m$ truck groups starting from different depots in set $\mathcal{P}\subseteq\mathcal{V}$ to visit a set of $n$ customers $\mathcal{C}\subseteq\mathcal{V}$, find a set of $m$ truck group routes $\{\pi_p\}_{p\in\mathcal{P}}$, one for each truck group, to visit every customer $c\in\mathcal{C}$ exactly once that minimizes the total cost $\sum_{p\in\mathcal{P}} \text{cost}(\pi_p)$.   
\end{problem}

\section{Mathematical formulation}
\label{sec:milp}
According to the problem defined in \autoref{sec:problem}, this section proposes the MILP model for MA-FSTSP. 
\subsection{Notations}
Before formulating the model, we summarize the parameters and decision variables in \autoref{tab: param} and \autoref{tab: var}. 
\begin{table}[ht]
    \caption{Parameter notations}
    \label{tab: param}
    \centering
    \begin{tabular}{p{0.1\textwidth}p{0.83\textwidth}}
    \toprule
    Parameter & Description \\
    \midrule
    $\mathcal{C}$   & Set of customers\\
    $\mathcal{P}$   & Set of depots\\
    $\mathcal{V}$   & Set of vertices in the road network \\
    $\mathcal{D}_p$ & Set of drones carried by the truck from depot $p$\\
    $d(u, v)$ & Geometry distance between vertex $u$ and vertex $v$ \\
    $d^{\text{tr}}(u, v)$ & Shortest path distance from vertex $u$ to vertex $v$ along road network\\
    $T$ & Maximum possible tour length for one truck\\
    $M$ & A sufficiently large positive number \\
    \bottomrule
    \end{tabular}
\end{table}
\begin{table}[ht]
    \caption{Variables notations}
    \label{tab: var}
    \centering
    \begin{tabular}{p{0.1\textwidth}p{0.83\textwidth}}
    \toprule
    Variable & Description\\
    \midrule
    $\alpha_{p, c}$ & Equals 1 if customer $c$ is visited by the truck group from depot $p$; 0 otherwise \\
    $e_{p, u, v, t}$ & Equals 1 if the truck route for the truck from depot $p$ has $v_{t-1}=u$ and $v_t=v$, 0 otherwise\\
    $\beta_c$ & Equals 1 if customer $c$ is visited by drone, 0 if visited by truck \\
    $x_{u, c, t}$ & Equals 1 if the drone visiting customer $c$ take off from the truck at truck route vertex $v_{t-1}=u$, 0 otherwise \\
    $y_{u, c, t}$ & Equals 1 if the drone visiting customer $c$ synchronize with truck at truck route vertex $v_t=u$, 0 otherwise\\
    $l_{p, t}$ & Number of drones left on the truck from depot $p$ at $v_t$\\
    $\tau_{p, u, t}$ & Arrival time at vertex $v_t=u$ for truck from depot $p$\\
    \bottomrule
    \end{tabular}
\end{table}

\subsection{Mathematical model}
\begin{subequations}\label{milp:prob}
\begin{flalign}
    \min\ &\sum_{p\in \mathcal{P}} \tau_{p,p,T}\label{milp:obj}\\
    s.t.\ & \sum_{p\in\mathcal{P}}\alpha_{p, c}=1, & c\in\mathcal{C},\label{milp:part}\\
    & \sum_{u\in\mathcal{V}}|\mathcal{C}|\cdot e_{p, p, u, 0} \ge \sum_{c\in\mathcal{C}}\alpha_{p, c}, & p\in\mathcal{P},\label{milp:edge_start}\\
    & \sum_{u, v\in\mathcal{V}}e_{p, u, v, t}\le 1, & p\in\mathcal{P}, t\le T,\\
    & \sum_{u, v\in\mathcal{V}}e_{p, u, v, t + 1}\le \sum_{u, v\in\mathcal{V}}e_{p, u, v, t}, & p\in\mathcal{P}, t\le T,\\
    & \sum_{u\in\mathcal{V}}\sum_{t\le T} e_{p, u, v, t}=\sum_{u\in\mathcal{V}}\sum_{t\le T}e_{p, v, u, t}, & p\in\mathcal{P}, v\in\mathcal{V}\\
    & \sum_{u\in\mathcal{V}}e_{p, u, v, t} = \sum_{u\in\mathcal{V}} e_{p, v, u, t + 1}, & p\in\mathcal{P}, v\in\mathcal{V}, t\le T-1,\label{milp:edge_end}\\
    & \sum_{u\in\mathcal{V}}\sum_{t\le T}x_{u, c, t} = \beta_{c}, & c\in\mathcal{C},\label{milp:take_off}\\
    & \sum_{u\in\mathcal{V}}\sum_{t\le T}y_{u, c, t} = \beta_{c}, & c\in\mathcal{C},\label{milp:land}\\
    & \sum_{u\in\mathcal{V}}\sum_{s\le t}x_{u, c, s}\ge \sum_{u\in\mathcal{V}}\sum_{s\le t}y_{u, c, t}, & c\in\mathcal{C}, t\le T,\label{milp:order}\\
    &\sum_{u\in\mathcal{V}}\sum_{t\le T}x_{u, c, t} \cdot d(u, c) + y_{u, c, t}\cdot d(c, u)\le r, & c\in\mathcal{C},\label{milp:limit}\\
    & \sum_{t\le T}\sum_{u\in\mathcal{V}}e_{p, c, u, t} - \sum_{t\le T} e_{p, c, c, t} \ge 1 - \beta_c - (1-\alpha_{p, c})\cdot M, & p\in\mathcal{P},  c\in\mathcal{C},\label{milp:truck_visit}\\
    & \sum_{u\in\mathcal{V}}|\mathcal{D}_p|\cdot e_{p, v, u, t} \ge \sum_{c\in\mathcal{C}}x_{v, c, t} \cdot \alpha_{p, c}, &p\in\mathcal{P}, v\in\mathcal{V}, t\le T,\label{milp:sync_takeoff}\\
    &\sum_{u\in\mathcal{V}}|\mathcal{D}_p|\cdot e_{p, u, v, t} \ge \sum_{c\in\mathcal{C}}y_{v, c, t}\cdot \alpha_{p,c}, &p\in\mathcal{P}, v\in\mathcal{V}, t\le T,\label{milp:sync_land}\\
    & \sum_{v\in\mathcal{V}}\sum_{t\le T}2|\mathcal{C}|\cdot e_{p, u, v, t}-\sum_{t\le T}2|\mathcal{C}|\cdot e_{p, u, u, t}\ge \sum_{c\in\mathcal{C}}\sum_{t\le T}x_{u, c, t}\cdot \alpha_{p,c} +y_{u, c, t}\cdot \alpha_{p,c}, & p\in\mathcal{P}, u\in\mathcal{V},\label{milp:cycle}\\
    & l_{p,1} = |\mathcal{D}_p| - \sum_{c\in\mathcal{C}}x_{p, c, 1}\cdot \alpha_{p, c}, & p\in\mathcal{P},\label{milp:drone_start}\\
    & l_{p, T} = |\mathcal{D}_p|, & p\in\mathcal{P},\\
    & l_{p, t}\le |\mathcal{D}_p|, &p\in\mathcal{P}, t\le T,\\
    & l_{p, t} - l_{p, t-1} = \sum_{u\in\mathcal{V}}\sum_{c\in\mathcal{C}}y_{u, c, t}\cdot \alpha_{p,c} - x_{u, c, t}\cdot \alpha_{p,c}, & p\in\mathcal{P}, 2\le t\le T \label{milp:drone_end}\\
    &\tau_{p,u,t}\le \tau_{p,u, t+1}, &p\in\mathcal{P}, u\in\mathcal{V}, t\le T - 1,\label{milp:time_consist}\\
    & \tau_{p,u,t}\ge \tau_{p,v,t-1}+e_{p, v, u, t}\cdot d^{\text{tr}}(v, u) / s^{\text{tr}} - (1 - e_{p,v,u, t})\cdot M, & p\in\mathcal{P}, u, v\in\mathcal{V}, 2\le t\le T,\label{milp:time_truck}\\
    & \tau_{p, u, t}\ge \tau_{p, v, s} + x_{v, c, s + 1}\cdot\alpha_{p, c}\cdot d(v, c) / s^{\text{dr}}+ y_{u, c, t}\cdot\alpha_{p, c}\cdot d(c, u)/ s^{\text{dr}}\nonumber\\
    &\phantom{\tau_{p, u, t}\ge \tau_{p, v, s}}-(2-x_{v, c, s+1}-y_{u,c, t})\cdot M, &p\in\mathcal{P}, u,v\in\mathcal{V}, c\in\mathcal{C}, s<t\le T.\label{milp:time_drone}
\end{flalign}
\end{subequations}
The objective function \ref{milp:obj} minimizes the overall time cost to visit all customers and return to the depots. Constraint \ref{milp:part} assigns each customer to exactly one truck group. Constraints \ref{milp:edge_start}-\ref{milp:cycle} ensure the truck route of each truck group is a valid cycle as people did in TSP. Constraints \ref{milp:take_off}-\ref{milp:land} select the take-off and land vertices for customers drones visit. Constraint \ref{milp:order} ensures the correct drone take-off and landing order. Constraint \ref{milp:limit} limits the flying distance of drones for each customer visit. Constraint \ref{milp:truck_visit} ensures trucks to visit their assigned customers. Constraints \ref{milp:sync_takeoff} and \ref{milp:sync_land} synchronize the truck and drones at the selected take-off and land vertices. Constraints \ref{milp:drone_start}-\ref{milp:drone_end} track and constrain the number of available drones throughout the route. Constraint \ref{milp:time_consist} maintains the time consistency along the route. Constraint \ref{milp:time_truck} accounts for the truck travel time, while constraint \ref{milp:time_drone} accounts for the drone travel time. 

The product term of binary variables $x_{u,c,t}\cdot\alpha_{p, c}$ can be linearized as a new binary variable $z_{u, c, t}$ which satisfies $z_{u,c, t}\ge x_{u,c,t}+\alpha_{p, c}-1$, $z_{u,c,t}\le x_{u,c,t,}$, and $z_{u,c,t}\le \alpha_{p, c}$. The term $y_{u,c,t}\cdot\alpha_{p, c}$ can be linearized similarly, so the model is a MILP. For readability, the product terms are kept in the constraint formulations.

\section{Methodology}
\label{sec:methodology}
\begin{algorithm}[!htb]
    \caption{Framework to solve MA-FSTSP}\label{alg:method}
    \textbf{Input:} Road network $\mathcal{G}$, customer set $\mathcal{C}$, depot $\mathcal{P}$, truck speed $s_{\text{tr}}$, drone speed $s_{\text{dr}}$, drone endurance $r$, pairwise distance $d$
    \begin{algorithmic}[1]
        \State Initialize solution $\textsc{Soln}$ as empty set 
        \For{$(u, v)\in\mathcal{C}\times\mathcal{C}$}
        \State $d^{\text{tr}}(u,v)\leftarrow\text{ShortestPathLength}(u,v;\mathcal{G})$
        \EndFor
        \State Assign customers $\{\mathcal{C}_i\}_{i=1}^{|\mathcal{P}|}\leftarrow\text{Partition}(\mathcal{C},\mathcal{P};d^{\text{tr}}, d)$
        \For{$i\in\{1,2,\cdots,m\}$}
        \State $\mathcal{S}\leftarrow\{\{n\in\mathcal{V}(\mathcal{G}):w(n, c)<r/2\}:c\in\mathcal{C}_i\}\cup\{\{p_i\}\}$ \Comment{For each customer $c$ collect nearby vertices}
        \State \textsc{Order} $\leftarrow$ SetTSP($\mathcal{S};d^{\text{tr}}, d;s_{\text{tr}}, s_{\text{dr}}$)
        \State \textsc{Soln}$[i]\leftarrow\text{GetRoute}(\mathcal{C}_i, p_i;\textsc{Order};d^{\text{tr}}, d;s_{\text{tr}}, s_{\text{dr}})$
        \EndFor
        \State \textbf{return} \textsc{Soln}
    \end{algorithmic}
\end{algorithm}

We propose a three-phase algorithm to solve MA-FSTSP, as shown in \autoref{alg:method}. In phase 1, the problem is decomposed to $m$ subproblems (recall that $m$ is the number of depots) of a single truck carrying multiple drones to visit customers on road networks by allocating customers to truck groups via set-based partition algorithms. Next, in phase 2, a set TSP heuristic is applied to determine the visiting order of customers for each truck group. Finally, in phase 3, a truck-drone route is computed given the restriction of customers' visiting orders. The pseudocode for the proposed algorithm is shown in \autoref{alg:method}. Line 1$\sim$4 is the preparation for data. Line 5 corresponds to phase 1. Line 7 computes the set for each customer, and lines 8$\sim$9 correspond to phases 2 and 3, respectively. Methods of phase 1$\sim$3 are introduced in \autoref{sec:phase1}, \autoref{sec:phase2}, and \autoref{sec:phase3}, respectively. An acceleration method for Set-TSP is introduced in \autoref{sec:acc}. The computational complexity analysis of each phase is detailed in \ref{app: tc}.

\subsection{Phase 1: Set-based partition algorithms}\label{sec:phase1}
The first phase of our proposed algorithm focuses on efficiently assigning customers to truck groups. This section introduces a novel set-based extension approach that addresses the limitations of existing assignment methods for truck-drone delivery systems.

Commonly used strategies for customer assignment in multi-agent systems include the nearest-neighbor algorithm~\cite{ho2008hybrid, salhi2014multi, geetha2012metaheuristic} and the minimum spanning tree partition algorithm~\cite{yang2024hierarchical}. However, these methods often fall short in practical scenarios where geometric proximity does not accurately reflect travel time or cost. For example, two geographically close vertices in practice might be distantly connected in road networks, commonly seen in practical scenarios such as places around highways and rivers, one-way roads, and traffic bottlenecks. In these cases, neither geometric nor road map distance accurately estimates the time required for a truck group to travel between customers. 

To address these limitations, we propose a set-based distance metric that more accurately approximates the point-wise distance for truck groups. This approach extends the nearest-neighbor and minimum spanning tree partition algorithms to their set versions.

Inspired by \citet{marcucci2021shortest}, we define a set of neighboring vertices for each customer $c \in \mathcal{C}$ as follows:
\begin{equation}\label{eq: neighbor}
    \mathcal{S}_c(\theta)\coloneqq\{n\in\mathcal{V}(\mathcal{G}):d(n,c)\le\theta\}
\end{equation} 
where $\theta$ is a distance parameter, and $d(n,c)$ represents the distance between vertex $n$ and customer $c$. It is worth noting that if $\theta \le r$, each vertex in $\mathcal{S}_c$ serves as a valid takeoff or landing point for a drone visiting customer $c$. Using this set, we define the distance from set $\mathcal S_c(\theta)$ to set $\mathcal S_{c'}(\theta)$ as
\begin{equation}\label{eq:set}
    d^{\text{set}}(\mathcal S_c(\theta), \mathcal S_{c'}(\theta)) \coloneqq \min_{v\in\mathcal{S}_c(\theta), v'\in\mathcal{S}_{c'}(\theta)}\left\{d(c, v)/s_{\text{dr}} + d^{\text{tr}}(v, v') /s_{\text{tr}}+d(v', c')/s_{\text{dr}}\right\}\cdot s_{\text{tr}},
\end{equation}
where $d^\text{tr}$ is the shortest path length for the truck, and $s_{\text{dr}}$ and $s_{\text{tr}}$ are the speeds of the drone and truck, respectively. This distance metric is normalized based on the time required for a truck group to visit customer $c'$ after visiting customer $c$. Since the vertex $p$ can also be viewed as a set $S_p(0^+)$ containing $p$ only, we extend the domain of $\mathcal{S}_c$ from $c\in\mathcal{C}$ to $c\in\mathcal{C}\cup\mathcal{P}$. Without ambiguity, we omit $\theta$ and abbreviate the distance as $d^{\text{set}}(c,c')$ for $c, c'\in\mathcal{C}\cup\mathcal{P}$. 

\begin{algorithm}[!htb]
    \caption{Pseudocode for Set-Based Paritition}\label{alg:snn}
    \textbf{Input:} Road network $\mathcal{G}$, customer set $\mathcal{C}$, depot $\mathcal{P}$, truck speed $s_{\text{tr}}$, drone speed $s_{\text{dr}}$, radius $\theta$, pairwise distance $d$
    \begin{algorithmic}[1]
        \For{$c\in\mathcal{C}$}
        \State $\mathcal{S}_c(\theta)\leftarrow\{v\in\mathcal{V}(\mathcal{G}):d(v, c) < \theta\}$
        \EndFor
        \For{$p\in\mathcal{P}$}
        \State $\mathcal{S}_p\leftarrow\{p\}$
        \EndFor
        \State Construct fully connected graph $\mathcal{G}'(\mathcal{C}\cup\mathcal{P})$ using edge weight $w(c, c')=d^\text{set}(c, c')$
        \Comment{Using Eq.~\ref{eq:set}}
        \State $\{\mathcal{C}_p\}_{p\in\mathcal{P}}\leftarrow$NearestNeighbor($\mathcal{P}, \mathcal{C}$) or MinimumSpanningTreePartition($\mathcal{P},\mathcal{C}$)
        \State \textbf{return} assignment $\{\mathcal{C}_p\}_{p\in\mathcal{P}}$
    \end{algorithmic}
\end{algorithm}

The pseudocode for this assignment is given in \autoref{alg:snn}. We compute the pairwise $d^{\text{set}}$ (line 1-6) to construct a fully connected graph (line 7), on which we run the NN or MST partition algorithm to get the assignment (line 8). 

Next, each group finds their own shortest tours in phase 2-3. 

\subsection{Phase 2: Set traveling salesman problem heuristic}
\begin{figure}
    \centering\includegraphics[width=0.8\linewidth]{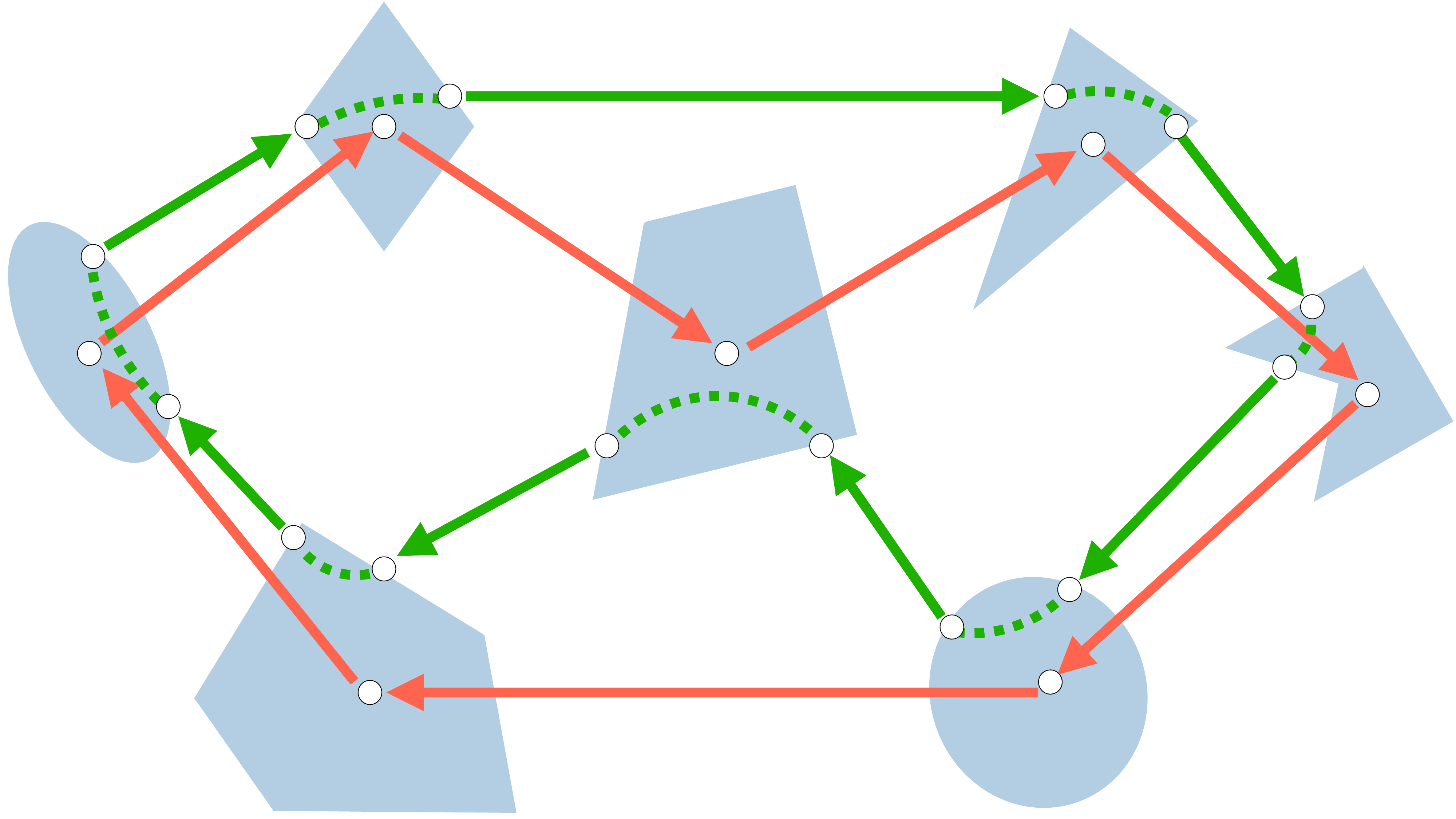}
    \caption{An example of Set TSP versus TSP. Each set represents a set of vertices in the graph, which can be arbitrarily connected inside. Sets are fully connected via vertices at the boundary. The red arrows represent the optimal TSP tour to visit the center vertices of sets in the graph, and the green arrows represent the optimal Set TSP tour to visit every set. Since Set TSP does not require visiting any exact vertex, it can pass through the set via \textbf{shortcuts} represented in dashed green. As a result, the visiting order of sets on the TSP tour and Set TSP tour can be different.}
    \label{fig: stsp}
\end{figure}
\label{sec:phase2}
Following the same insight in \autoref{sec:phase1}, we extend the TSP heuristic to a set version as well to take the usage of drones into account, which is formally stated as:
\begin{problem}[\textbf{Set TSP}]
    Given $n$ location sets $V_1, V_2,\cdots, V_n$ and the traveling costs between each pair of locations, find the route with the lowest total cost that visits each location set exactly once and returns to the starting location set. A visit to a location set is defined as consecutively visiting one or more locations within that set.
\end{problem}

An example of Set TSP versus TSP is shown in \autoref{fig: stsp}. For the truck group $\mathcal{T}_i$, approximating the task of designing a tour that visits each customer in $\mathcal{C}_i$, starting and ending at depot $p_i$, can effectively be modeled by a set TSP as follows. (1) The location sets are $\{S_c(\theta): c\in \mathcal C_i\}\cup\{\{p_i\}\}$, where $S_c(\theta)$ is defined in \autoref{eq: neighbor}. We choose $\theta=r/2$ so that all pairs of vertices in $S_c(\theta)$ can be the start and end vertices of a drone delivery task. (2) The traveling cost from $u\in S_c(\theta)$, $c\in C_i$ to $v$ is defined as: 
\begin{equation}\label{eq: cost}
   w(u,v)=\begin{cases}
       \min\Big\{\max\Big\{\frac{d(u,c) + d(c, v)}{s_\text{dr}}, d(u, v)/s_\text{tr}\Big\}, \frac{d^{\text{tr}}(u, c)+d^{\text{tr}}(c, v)}{s_{\text{tr}}}\Big\}, & v\in S_c(\theta);\\
       d(u,v)/s_\text{tr}, &v\notin S_c(\theta).
   \end{cases} 
\end{equation}
The cost is the time the truck group needs to move from vertex $u$ to vertex $v$ and visit customer $c$ if $u,v\in S_c(\theta)$. The visiting order of customers in the optimal route found for the set TSP is used as the heuristic in Phase 3.


The Set TSP is an NP-hard problem since the TSP is a special case when each set only contains one vertex, and the total cost can be validated in polynomial time. To solve it, we extend the Gavish-Graves (GG) MILP~\cite{gavish1978travelling} for TSP. 

Let $\beta_{c, c'}\in\{0,1\}$ be a binary variable indicating whether the next customer or depot is $c'$ after visiting customer or depot $c$, and $y_{c, c'}\in\mathbb{R}_+$ be the corresponding network flow. The TSP property of visiting each set exactly once can be expressed as 
\begin{subequations}\label{milp:tsp}
\begin{flalign}
    &\beta_{c, c} = y_{c, c} = 0, &\forall  c\in\mathcal{C}_i\cup\{p_i\},\label{cstr:selfloop}\\
    &\sum_{c'\in\mathcal{C}_i\cup\{p_i\}}\beta_{c, c'}=\sum_{c'\in\mathcal{C}_i\cup\{p_i\}}\beta_{c', c}=1, &\forall c \in\mathcal{C}_i\cup\{p_i\},\label{cstr:circle}\\
    &y_{c, c'}\le |\mathcal{C}_i|\cdot\beta_{c, c'}, &\forall c, c'\in\mathcal{C}_i\cup\{p_i\},\label{cstr:upper}\\
    &\sum_{c\in\mathcal{C}_i\cup\{p_i\}}y_{p_i, c}=|\mathcal{C}_i|,\\
    &\sum_{c\in\mathcal{C}_i\cup\{p_i\}}y_{c, p_i}=0,\\
    &\sum_{c'\in\mathcal{C}_i\cup\{p_i\}}y_{c', \bar c}-\sum_{c'\in\mathcal{C}_i\cup\{p_i\}}y_{c, c'}=1, &\forall c\in\mathcal{C}_i\cup\{p_i\}\label{cstr:consecutive}.
\end{flalign}
\end{subequations}
The constraint~\ref{cstr:selfloop} forbids self-loops, \ref{cstr:circle} forces the solution to be a collection of cycles, and \ref{cstr:upper}$\sim$\ref{cstr:consecutive} ask the flow to reduce 1 unit each step along the cycle from $|\mathcal{C}_i|$ to 0, which only allows the existence of one cycle.

Let $\gamma_{u, v}$ be a binary variable indicating whether a drone tour starts from $u$ to $v$ for $u, v\in\mathcal{S}_c$ and $\delta_{v, u}$ be a binary variable indicating whether a truck picks up a drone at $v$ and carry it to $u$ to dispatch it. The constraints are:
\begin{subequations}\label{milp:set}
\begin{flalign}
    &\sum_{u, v\in\mathcal{S}_{c}}\gamma_{u,v} = 1, &\forall c\in \mathcal{C}_i,\label{cstr:visit}\\
    &\sum_{v\in\mathcal{S}_{c}} \sum_{u\in\mathcal{S}_{c'}}\delta_{v, u} = \beta_{c, c'}, &\forall c, c'\in\mathcal{C}_i\cup\{p_i\},\label{cstr:align}\\
    &\sum_{u\in\mathcal{S}_{c}}\gamma_{u, v} = \sum_{c'\in\mathcal{C}_i\cup\{p_i\}}\sum_{w\in \mathcal{S}_{c'}}\delta_{v, w}, &\forall c\in\mathcal{C}_i\cup\{p_i\},v\in\mathcal{S}_{c}\label{cstr:in}\\
    &\sum_{v\in\mathcal{S}_{c}}\gamma_{u, v} = \sum_{c'\in\mathcal{C}_i\cup\{p_i\}}\sum_{w\in\mathcal{S}_{c'}}\delta_{w, u}, &\forall c\in\mathcal{C}_i\cup\{p_i\}, u\in\mathcal{S}_{c}\label{cstr:out}.
\end{flalign}
\end{subequations}
Constraint~\ref{cstr:visit} encodes the TSP constraint, which asks for exactly one drone visit starting from a $u\in S_c$ and ending at a $v\in S_c$ to each customer $c\in\mathcal C_i$.
Constraint~\ref{cstr:align} aligns the vertex-level route with the set-level route. If $c$ is visited followed by the visit of $c'$, i.e., $\beta_{c, c'}=1$, then the truck route should pass from a $v\in S_{c}$, where it picks up the drone visiting $c$, to a $u\in S_{c'}$, where it dispatches the drone to visit $c'$. Otherwise, such a sub-route does not exist in the truck route. Constraint~\ref{cstr:in} means that if the drone lands at $v\in S_{c}$ after visiting $c$, it is picked up at $v$ by the truck. It synchronizes the truck and the drone at the dispatching vertex. Constraint~\ref{cstr:out}, similarly, synchronizes the truck and the drone at the landing vertex.

The objective function minimizes the sum of costs of the edges (defined in \autoref{eq: cost}) traversed by the route
\begin{equation}\label{eq: obj}
\begin{aligned}
\sum_{c\in\mathcal{C}_i\cup\{p_i\}}\sum_{u\in\mathcal{S}_c}\left(\sum_{v\in\mathcal{S}_c}w(u,v)\cdot\gamma_{u, v} + \sum_{c'\in\mathcal{C}_i\cup\{p_i\}}\sum_{v\in\mathcal{S}_{c'}} w(u,v)\cdot \delta_{u, v}\right).
\end{aligned}
\end{equation}
$\gamma$ and $\delta$ together form the adjacent matrix of the truck group tour on the fully connected graph $\mathcal{G}(\{p_i\}\cup_{c\in\mathcal{C}_i}S_c(\theta))$, from which we can extract the order of customers to visit. 
\subsection{Phase 3: Decode solution from heuristic}
\label{sec:phase3}
Even when the customers' visiting order is given, determining the optimal route for the truck group is an NP-hard problem.
To approximate the optimal solution in polynomial time, we only consider the action of dispatching all drones simultaneously. Intuitively, when using multiple drones outperforms using one drone, the customers are close to each other within $O(r)$ distance, which means that the assumption does not sacrifice the performance a lot if $r$ is small.

Based on the assumption above, dynamic programming can finish the decoding procedure within a polynomial time. Given visiting order $\mathcal{O}_i$ for truck group $\mathcal{T}_i$, we first define $\textsc{Time}(u, v; s, t)$ to be the minimum time needed for the truck group to visit customers $\{\mathcal{O}_i^s, \mathcal{O}_i^{s+1},\cdots,\mathcal{O}_i^{s+t-1}\}$ with $t$ drones, which are dispatched together from $u$, and the last one is collected at vertex $v$. The initial state $\textsc{Time}(u, v; s, 1)$ is the time for the truck with one drone to visit the customer $\mathcal{O}_i^{s}$ starting from $u$ to $v$, i.e., \begin{align}\label{eq: init}
    &\textsc{Time}(u, v; s, 1)=\begin{cases}
        \frac{d^{\text{tr}}(u, \mathcal{O}_i^s) + d^{\text{tr}}(\mathcal{O}_i^s, v)}{s_{\text{tr}}},& d(u, \mathcal{O}_i^s) + d(\mathcal{O}_i^s, v) > r;\\
        \max\{\frac{d(u, \mathcal{O}_i^s) + d(\mathcal{O}_i^s, v)}{s_{\text{dr}}},
        \frac{d^{\text{tr}}(u, v)}{s_{\text{tr}}}\}, & \text{otherwise}.
    \end{cases}
\end{align}
In the first case, the drone cannot finish the delivery task due to the energy limit given the position of $u$, $v$, and $O_i^s$. We set $\textsc{Time}(u,v;s, 1)$ as the time needed for the truck to finish the delivery task. In the second case, where the drone can finish the delivery task, we set $\textsc{Time}(u,v;s, 1)$ as the minimum time in which the truck can move from $u$ to $v$ along the road network and the drone can fly from $u$ to $O_i[s]$ and then to $v$. For $t\ge 2$, if the $t$-th drone can finish the delivery task, i.e., $d(u, \mathcal{O}_i^{s + t - 1}) + d(\mathcal{O}_i^{s + t - 1}, v) \le r$, we compute the minimum time needed to finish the $t$ delivery tasks by identifying the optimal vertex $w$ to collect the $(t-1)$-th drone. This minimizes the truck's overall time to travel from $u$ to $v$ while ensuring that the first $t-1$ drones are collected.  
\begin{align}\label{eq: transit}
&\textsc{Time}(u, v; s, t)=\min\Bigg\{
    \max_{w\in S_{\mathcal{O}_i^{s+t-2}}(r)}
    \textsc{Time}(u, w; s, t - 1)+\frac{d^{\text{tr}}(w,v)}{s_{\text{tr}}}, \frac{d(u, \mathcal{O}_i^{s + t - 1}) + d(\mathcal{O}_i^{s + t - 1}, v)}{s_{\text{dr}}}\Bigg\}.
\end{align}
Otherwise, we set $\textsc{Time}(u,v;s,t)$ to be $+\infty$.

Then, we define $\textsc{Value}(s, u)$ as the optimal cost starting from the depot $p_i$ to the current vertex $u$ after serving the first $s$ customers. The initial state 
\begin{equation}
    \textsc{Value}(0, v) = d^{\text{tr}}(p_i, v),\forall v\in\mathcal{V}(\mathcal{G}),
\end{equation}
which is the adjusted shortest path length from depot $p_i$ to end vertex $u$. Given the number of drones $k$, the transition function takes the $k$ preceding states into account,
\begin{equation}
    \textsc{Value}(s, v) = \max_{0\le t\le k,\ u, w\in\mathcal{V}(\mathcal{G})}\textsc{Value}(s - t, u) +\textsc{Time}(u, w; s - t + 1, t) + d^{\text{tr}}(w, v).
\end{equation}
$\textsc{Value}(|\mathcal{O}_i |, p_i)$ is the optimal cost for truck group $\mathcal{T}_i$, and the corresponding route can be reconstructed by tracing the maximum condition backward. 

Since the drone must take off and land within $\cup_{c\in\mathcal{O}}\mathcal{S}_c(r)$ due to the limit of battery capacity, we only need to initialize the vertices  \begin{equation}
    \textsc{Value}(0, v) = d^{\text{tr}}(p_i, v),\forall v\in\cup_{c\in\mathcal{O}}\mathcal{S}_c(r).
\end{equation}
Furthermore, for all the intermediate vertices $u\in\mathcal{V}$ that the truck travels through in the optimal route without dispatching or collecting drones, $\textsc{Value}(s, u) + \max_{w\in\mathcal{S}_{\mathcal{O}^{s}}(r)}\textsc{Time}(u, w; s-t + 1, t)+d^{\text{tr}}(w, v)$ yields the same value for all $v$ vertices along the optimal route after $u$. This value represents when the truck arrives at vertex $v$ after serving $s$ customers in the optimal route. This observation allows us to eliminate redundant computations. Specifically, we don't need to enumerate $u$ over $\mathcal{V}\setminus\mathcal{S}{O^{s-t}}$ or $w$ over $\mathcal{V}\setminus\mathcal{S}{O^{s}}$. Instead, we can compute the transition function more efficiently as follows: 
\begin{equation}
    \textsc{Value}(s, v) = \max_{0\le t\le k,\ u\in\mathcal{S}_{O^{s-t}}(r), w\in\mathcal{S}_{O^{s}}(r)}\textsc{Value}(s - t, u) +\textsc{Time}(u, w; s - t + 1, t) + d^{\text{tr}}(w, v).
\end{equation}

\subsection{Accelerate Set-TSP}\label{sec:acc}
\begin{figure}
    \centering
    \subfigure[Overlap]{
    \includegraphics[width=0.31\linewidth]{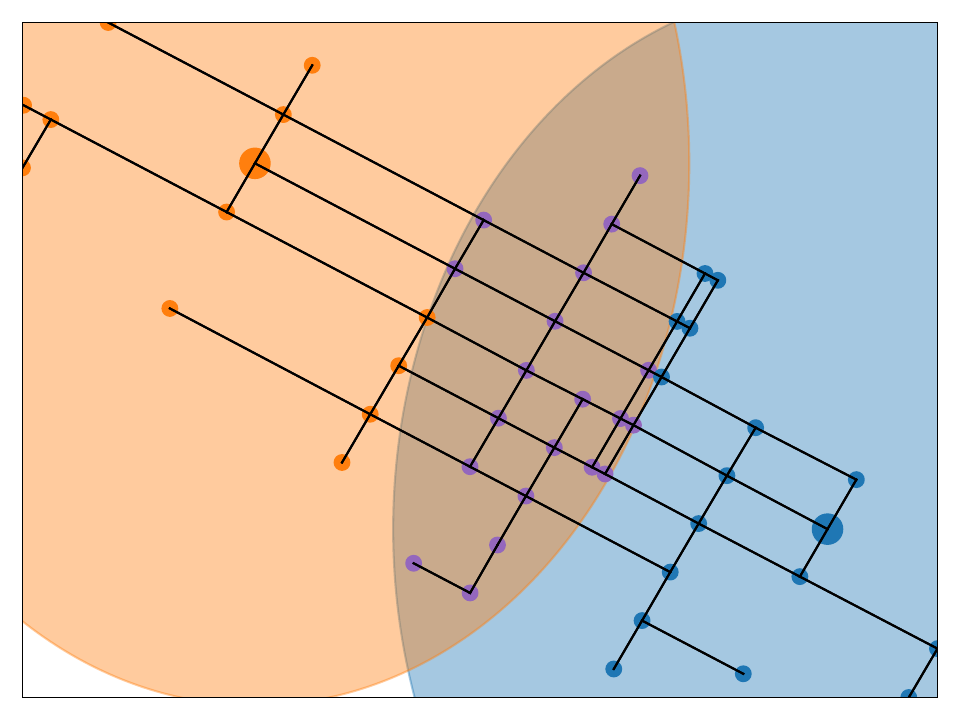}
    \label{fig: overlap}
    }
    \subfigure[Remove overlap]{
    \includegraphics[width=0.31\linewidth]{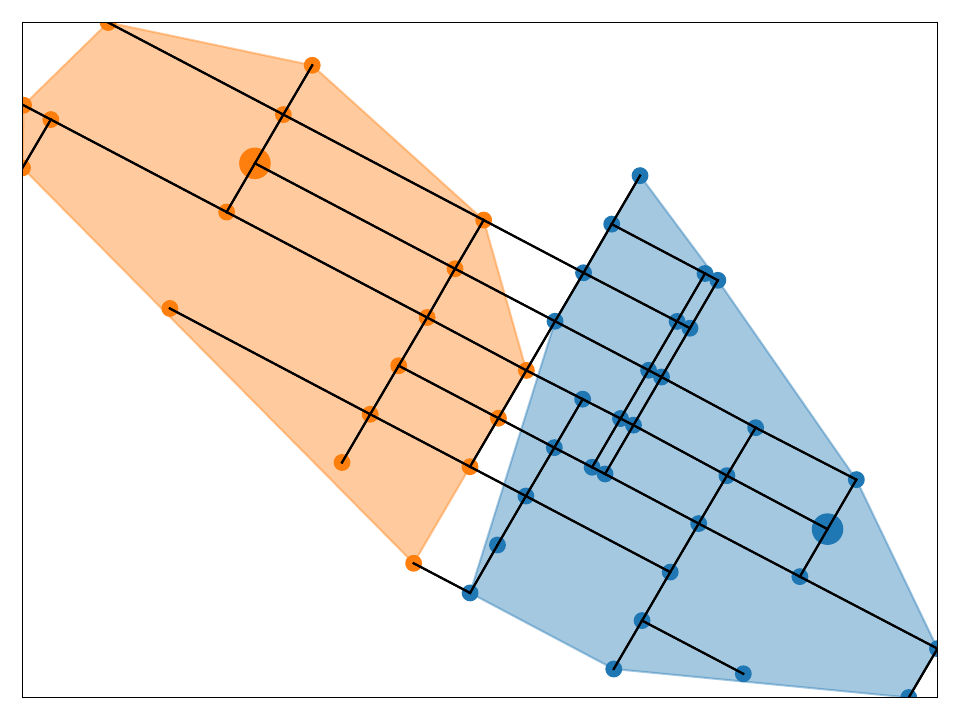}
    \label{fig: split}
    }
    \subfigure[Remove inner nodes]{
    \includegraphics[width=0.31\linewidth]{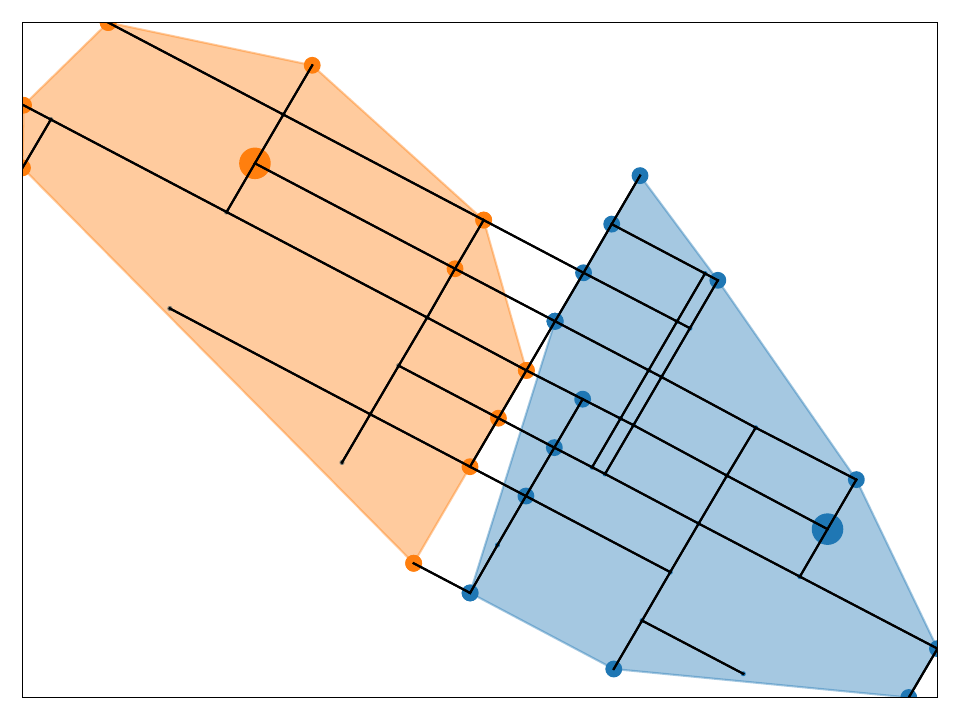}
    \label{fig: boundary}
    }
    \caption{Demonstration of acceleration strategies for the Set-TSP problem. (a) When two customers (the circle centers marked as larger vertices) are close, their neighbor sets $\mathcal{S}_c(r/2)$ overlap. Vertices appearing in both sets are highlighted in purple. (b) Our first strategy reduces redundancy by assigning vertices in the overlapped area to their nearest customer. (c) Our second strategy pre-computes optimal routes for all entering-leaving pairs on the boundary of each set, then removes all inner nodes from the neighbor sets. This significantly reduces the problem's complexity while maintaining solution quality.}
\end{figure}
The Set-TSP proposed in the second phase is solved $|\mathcal{P}|$ times in the problem, making its acceleration crucial for computational efficiency. We present acceleration strategies that significantly speed up the computation with minimal impact on solution quality. 

Our first strategy addresses the observation that as the number of customers increases, the sets of two nearby customers obtained in the second phase often overlap, as illustrated in \autoref{fig: overlap}. In such cases, some vertices appear in multiple sets. We eliminate this redundancy by assigning vertices in overlapped areas to their nearest customer, as shown in \autoref{fig: split}. This simplification reduces the overall complexity of the problem without significantly compromising solution quality.

The second strategy leverages the drone distance limit $r$. Regardless of whether a customer $c$ is visited by a truck or drone, the truck must enter the set $\mathcal{S}_c(r/2)$. Consequently, it will also have a vertex to leave the set at the boundary. Based on this insight, we pre-compute the optimal strategy between all pairs of vertices at the boundary of $\mathcal{S}_c(r/2)$, considering them as entering and leaving vertices for a truck group visiting customer $c$. We then remove all inner nodes in $\mathcal{S}_c(r/2)$, retaining only the boundary of each set. This approach significantly reduces the number of vertices to consider, from $O(r^2)$ to $O(r)$, substantially decreasing computational complexity while maintaining the integrity of the solution space.

To implement these strategies in combination, we first apply the overlap removal technique to each $\mathcal{S}_c(r/2)$ set as per the first strategy. Subsequently, following the second strategy, we eliminate the inner nodes in each resulting set. This integrated approach yields the final configuration as depicted in \autoref{fig: boundary}. The combined effect of these strategies substantially improves computational efficiency for the Set-TSP problem while almost preserving the quality of the solutions obtained.

\section{Experiments and analysis}
\label{sec:experiments}
In this section, we report experiments to validate our method. In \autoref{sec: compare}, we show the effectiveness and scalability by comparing our method with two baselines: a column generation-based method~\cite{gao2023scheduling} and a hill-climbing algorithm. We give the empirical optimality gap an upper bound by comparing it with a lower bound. In \autoref{sec: scale}, we explore the scalability of our approach by varying the problem size in different ways. We validate the set extension approach and the acceleration strategies in \autoref{sec: ablation} and \autoref{sec: acc}. In \autoref{sec: factor}, we do the sensitivity analysis to explore the influence of drone speed, distance limit for drones, and number of drones.

All experiments are performed on a server running on Ubuntu 20.04.6 LTS with an AMD Ryzen Threadripper 3990X 64-Core Processor. The MILP solver was Gurobi 10.0.2~\cite{gurobi}. The proposed models and algorithms are implemented in Python. The time limit for all test instances is set to 7200 seconds.  

\subsection{Comparison against baselines}\label{sec: compare}
\begin{table}[ht]
\caption{Comparison experiments of our algorithm, CG, and HC-VNS baselines. }
\label{tab: all}
\centering
\begin{tabular}{ccccccccc}
\toprule
& & & \multicolumn{2}{c}{CG} & \multicolumn{2}{c}{HC-VNS} & \multicolumn{2}{c}{Ours} \\ \cmidrule(r){4-5} \cmidrule(r){6-7}\cmidrule(r){8-9}
Maps & \#Customers & Lower Bound & Cost & Time & Cost  & Time & Cost  & Time \\ 
\midrule \\ 
\multirow{3}{*}{\begin{tabular}[c]{@{}c@{}}Manhattan \\ (Partial)\end{tabular}} 
& 5  & 0.52 & 3.69 & 727.81 & 0.87 & 0.93 & 0.88 & 0.07 \\
& 10 & 0.93 & 4.61 & 1259.44 & 1.51 & 1.69 & 1.41 & 0.22 \\
& 15 & 1.25 & 6.25 & 2026.17 & 1.98 & 2.43 & 1.77 & 0.24 \\
\\
\multirow{3}{*}{Manhattan} 
& 50 & - & - & - & 34.18 & 365.32 & 18.74 & 34.84\\
& 100 & - & - & - & 49.24 & 788.74 & 24.04 & 90.82\\
& 150 & - & - & - & 59.70 & 1339.71 & 27.89 & 193.71\\
\\
\multirow{3}{*}{Boston} 
& 50 & - & - & - & 106.13 & 686.25 & 72.57 & 97.65\\
& 100 & - & - & - & 146.53 & 2570.23 & 90.54 & 504.10\\
& 150 & - & - & - & 176.93 & 5039.28 & 99.64 & 1285.22\\
\bottomrule
\end{tabular}
\end{table}

\textbf{Baselines.} To our knowledge, few existing methods address the problem of multiple trucks and multiple drones on road networks. Notable approaches include a column generation-based heuristic method~\cite{gao2023scheduling}, a hybrid genetic algorithm, and a hybrid particle swarm algorithm~\cite{lin2022discrete}. The column generation-based heuristic formulates the partition problem as the master problem and the routing problem for each truck group as the pricing problem, approximated via Variable Neighborhood Search (VNS). The other two methods failed to solve the problem within the time limit on the smallest map, so we omit their results.  

We also implement a VNS-based method incorporating the widely used Hill-Climbing (HC) algorithm~\cite{chinnasamy2022review}. This method assigns customers using the nearest-neighbor approach, initializes truck group routes with truck-only routes, and iteratively improves the current route by selecting the best route in its neighborhood. The neighborhood is defined by applying one of the following operations to the current route: (1) modifying a drone's takeoff/landing location; (2) reassigning a customer from truck to drone service; (3) swapping the visiting order of two consecutive customers.

\textbf{Datasets.} We evaluate the methods on datasets sampled from real-world road networks of three different scales. The smallest road network is derived from a fully connected component of the Manhattan road map~\cite{blahoudek2020qualitative}, containing 20 vertices. The medium-sized network is the complete Manhattan map, comprising 1024 vertices. The largest network is the Boston road map obtained from OpenStreetMap~\cite{OpenStreetMap}, encompassing 11,000 vertices.
From each road network, we generate three datasets with a fixed number of depots and varying numbers of customers:
\begin{enumerate}
    \item Small network: 2 depots ($|\mathcal{P}|=2$), with 5, 10, and 15 customers.
    \item Medium network: 5 depots ($|\mathcal{P}|=5$), with 50, 100, and 150 customers.
    \item Large network: 10 depots ($|\mathcal{P}|=10$), with 50, 100, and 150 customers.
\end{enumerate}
The Boston road network is more dense than the Manhattan road network. The same set radius in each phase covers more nodes in Boston than in Manhattan, which makes it much harder even if the number of customers is the same. 

\textbf{Parameters.}  We set the number of drones on each truck to be 2, 3, and 4 on the Manhattan Partial, Manhattan, and Boston. The truck speed is set to be 30 km/h and the drone speed to be 48 km/h as \citet{gao2023scheduling}. The drone delivery distance limit $r$ is set at 1.5 km.

\textbf{Results.} For smaller problem instances, our proposed method demonstrates near-optimal performance. It achieves results that are remarkably close to the optimal solution, particularly compared to the Column Generation (CG) approach. This indicates the method's effectiveness in handling compact problem spaces. Also, our method exhibits exceptional speed across all dataset sizes. It consistently outperforms all baseline algorithms regarding execution time, making it particularly suitable for time-sensitive tasks for truck-drone delivery systems. Furthermore, as the problem size increases, the advantages of our method become more pronounced. It maintains its computational efficiency and surpasses all baseline approaches in solution quality. This scalability is crucial for addressing real-world applications, which sometimes contain hundreds of orders to deliver.

\begin{figure}[t]
    \centering
    \subfigure[Fix the number of depots $|\mathcal{P}|$]{
    \includegraphics[width=0.3\linewidth]{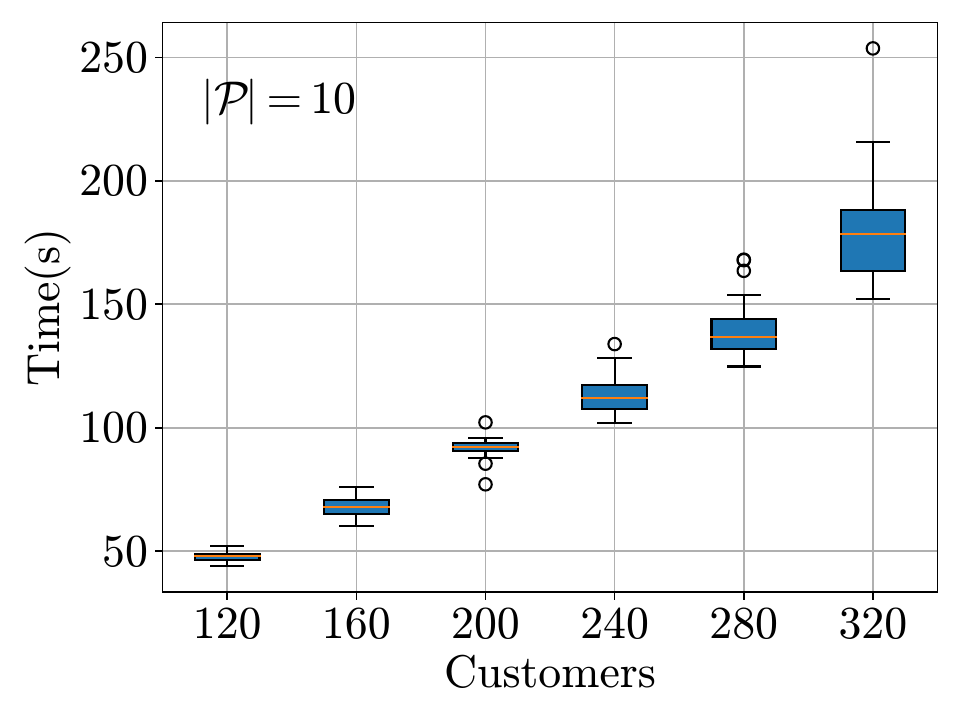}
    \label{fig: customer}
    }
    \subfigure[Fix the number of customers $|\mathcal{C}|$]{
    \includegraphics[width=0.3\linewidth]{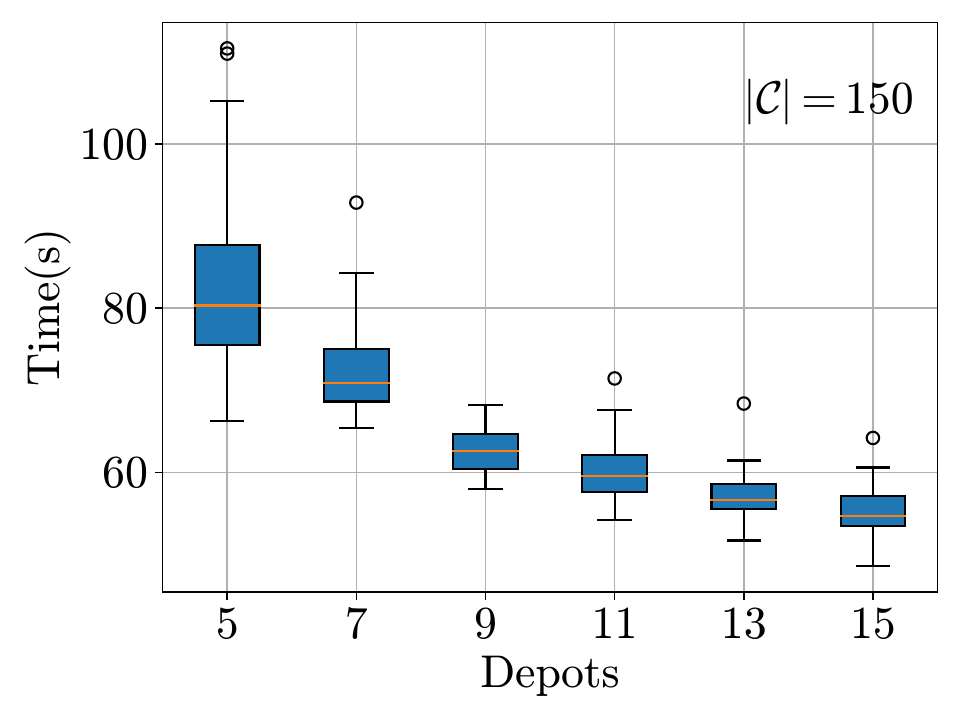}
    \label{fig: depot}
    }
    \subfigure[Fix the customer-depot ratio $|\mathcal{C}|/|\mathcal{P}|$]{
    \includegraphics[width=0.3\linewidth]{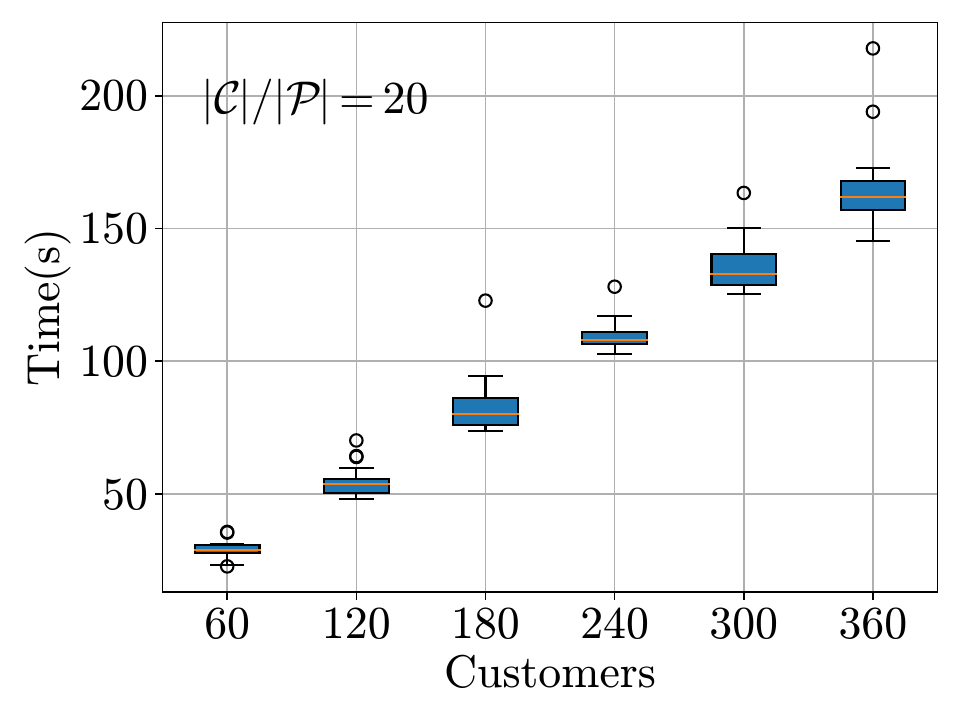}
    \label{fig: rate}
    }
    \caption{Box plots of scalability experiments.}
\end{figure}

\subsection{Scalability analysis}\label{sec: scale}
In this section, we evaluate the scalability of our proposed methods. For this analysis, we employ the set-NN algorithm in Phase 1 and the accelerated set-TSP algorithm in Phase 2. Our scalability tests involve varying two key parameters: the number of customers $|\mathcal{C}|$ and the number of depots (truck groups) $|\mathcal{P}|$. We conduct three distinct experiments: (1) fixed number of depots $|\mathcal{P}|$ with an increasing number of customers $|\mathcal{C}|$; (2) fixed number of customers $|\mathcal{C}|$ with an increasing number of depots $|\mathcal{P}|$; (3) fixed ratio of customers to depots $|\mathcal{C}|/|\mathcal{P}|$, simultaneously increasing both $|\mathcal{C}|$ and $|\mathcal{P}|$. 

The results of these experiments are presented in \autoref{fig: customer}, \autoref{fig: depot}, and \autoref{fig: rate}. Our findings demonstrate that the algorithm exhibits good scalability and can handle problems with up to 300 customers. The computational time patterns observed in each scenario are as follows: (1) when the number of depots is fixed, the computational time increases exponentially with the number of customers; (2) with a fixed number of customers, the computational time approximately follows a $1/{|\mathcal{P}|}$ relationship; (3) when maintaining a constant $|\mathcal{C}|/|\mathcal{P}|$ ratio, the computational time grows almost linearly as both $|\mathcal{C}|$ and $|\mathcal{P}|$ increase. These results indicate that the primary bottleneck in our pipeline lies in the second phase, specifically in solving the Set-TSP problem. This insight provides valuable direction for future optimizations and improvements to enhance the algorithm's overall performance and scalability.

\subsection{Effectiveness of set-based methods}\label{sec: ablation}
\begin{table}[ht]
  \caption{Results averaged over 100 instances sampled uniformly from Manhattan per scenario. Minimum cost marked bold.}
  \label{tab: theta}
  \centering
  \begin{tabular}{cccccccc}
    \toprule 
    && \multicolumn{2}{c}{$|\mathcal{P}|=5$, $|\mathcal{C}|=50$} &\multicolumn{2}{c}{$|\mathcal{P}|=5$, $|\mathcal{C}|=100$} &\multicolumn{2}{c}{$|\mathcal{P}|=5$, $|\mathcal{C}|=150$}\\
    \cmidrule(r){3-4}  \cmidrule(r){5-6}\cmidrule(r){7-8}
    Phase 1 & Phase 2 & Cost & Time(s) & Cost & Time(s) & Cost & Time(s) \\
    \midrule
    \multirow{2}{*}{NN} & TSP & 21.89 & 26.18 & 28.89 & 60.64 & 33.78 & 99.01\\
    & Set-TSP & 21.17 & 28.49 & 27.08 & 67.45 & 30.85 & 115.01\\
    \\
    \multirow{2}{*}{MST} & TSP & 19.51 & 27.05 & 25.82 & 62.59 & 30.94 & 100.80\\
    & Set-TSP & 19.00 & 29.95 & 24.21 & 72.18 & 28.05 & 126.87\\
    \\
    \multirow{2}{*}{Set-NN} & TSP & 21.38 & 26.52 & 28.31 & 61.17 & 33.01 & 99.31\\
    & Set-TSP & 20.55 & 31.14 & 26.15 & 77.02 & 29.78 & 141.29\\
    \\
    \multirow{2}{*}{Set-MST} & TSP & 19.43 & 27.09 & 25.71 & 62.35 & 30.87 & 100.83\\
    & Set-TSP & \textbf{18.74} & 34.84 & \textbf{24.04} & 90.82 & \textbf{27.89} & 193.71\\
    \bottomrule
  \end{tabular}
\end{table}
In this section, we evaluate the effectiveness of our set extension approach. We compare different methods in both phases of our algorithm. For the first phase, we examine the nearest neighbor (NN) method, the minimum spanning tree (MST) method, and their set extensions. In the second phase, we compare the standard and set TSP. 

To validate these methods, we tested all combinations of first and second-phase algorithms on three datasets. Each dataset contains 100 instances sampled uniformly from the Manhattan road network, consisting of 5 depots and 50, 100, or 150 customers.

The results of our analysis are presented in \autoref{tab: theta}. Our key findings are as follows. 

The combination of Set-MST and Set TSP produces the best cost but also requires the most computational time. Set TSP consistently provides lower costs than standard TSP across all first-phase methods, though it contributes the most to the increased computational time. This advantage of Set TSP over standard TSP becomes more pronounced as customers increase.

We found that MST outperforms NN in both its original and set versions. However, MST incurs higher computational costs. This is likely due to the more unbalanced allocation in the MST-based method, which leads to higher computational demands in the second phase due to the exponential time complexity.

Interestingly, when the second phase uses standard TSP, employing set-based methods in the first phase doesn't necessarily lead to better performance than their original counterparts. However, when Set TSP is used in the second phase, using set-based methods in the first phase does reduce the total cost.

These results highlight our algorithm's complex interplay between different methods and phases. They demonstrate the potential benefits of set-based approaches, particularly when used consistently across both phases, while also illustrating the trade-offs between solution quality and computational requirements.

\subsection{Effectiveness of acceleration strategies}\label{sec: acc}
\begin{table}[ht]
  \caption{Results averaged over 100 instances sampled uniformly from Manhattan per scenario. - means unable to solve within 7200 seconds.}
  \label{tab: acc}
  \centering
  \begin{tabular}{ccccccc}
    \toprule 
    & \multicolumn{2}{c}{$|\mathcal{P}|=5$, $|\mathcal{C}|=50$} &\multicolumn{2}{c}{$|\mathcal{P}|=5$, $|\mathcal{C}|=100$} &\multicolumn{2}{c}{$|\mathcal{P}|=5$, $|\mathcal{C}|=150$}\\
    \cmidrule(r){2-3}  \cmidrule(r){4-5}\cmidrule(r){6-7}
    Method & Cost & Time(s) & Cost & Time(s) & Cost & Time(s) \\
    \midrule
    No acceleration & 18.63 & 3668.46 & - & - & -& -\\
    Remove overlap & 18.69 & 71.15 & 23.93 & 129.50 & 27.77 & 227.21\\
    Remove inner vertices & 20.15 & 53.88 & 29.63 & 146.67 & 37.75 & 382.56\\
    Remove both & 18.74 & 34.84 & 24.04 & 90.82 & 27.89 & 193.71\\
    \bottomrule
  \end{tabular}
\end{table}

In this section, we evaluate the effectiveness of our acceleration strategies. For our experiments, we use Set-MST as the first phase algorithm and test four configurations in the second phase: no acceleration, strategy one only, strategy two only, and combined strategies. We conduct these experiments using the same datasets as in the previous section.

The results of our analysis are shown in \autoref{tab: acc}. When comparing the non-accelerated algorithm to the version applying both acceleration strategies, we observe only a 0.59\% sacrifice in solution quality with a significant speed-up of more than 100 times. By comparing strategies 1 and 2, we can see that most of this speed improvement is attributed to removing overlap. When applied alone, the strategy of eliminating inner vertices does not perform as well as expected. This is due to the overlap issue, which requires the truck to move back and forth between boundaries. As a result, it only becomes beneficial when no overlap present exists.

\begin{figure}[t]
    \centering
    \subfigure[Relative speed $s_{\text{dr}}/s_{\text{tr}}$]{
    \includegraphics[width=0.46\linewidth]{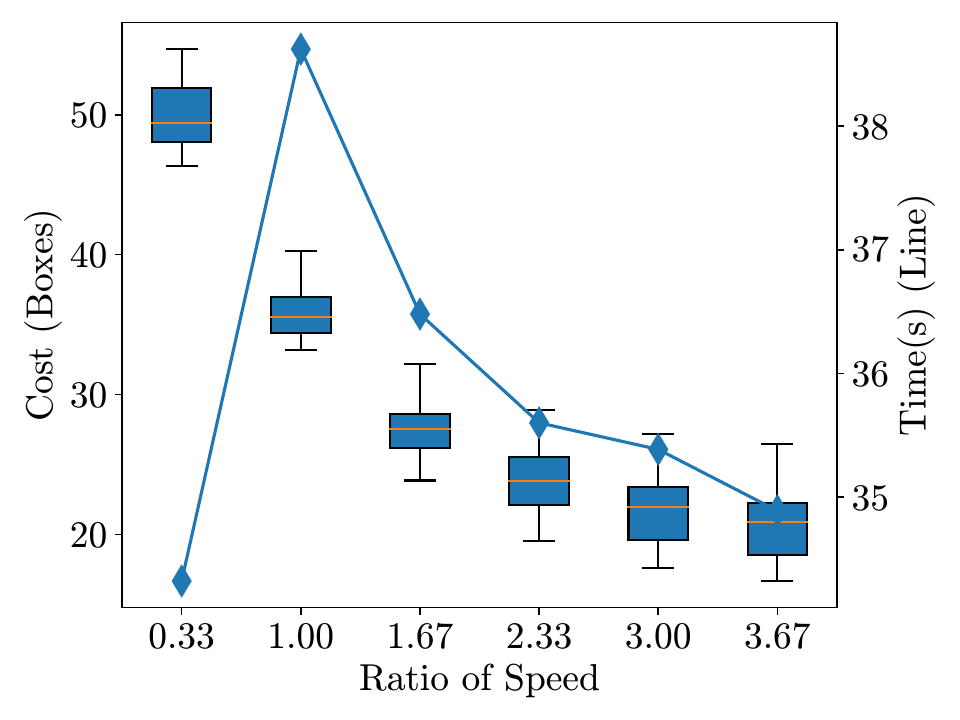}
    \label{fig: speed}
    }
    \subfigure[Distance limit $r$]{ 
    \includegraphics[width=0.46\linewidth]{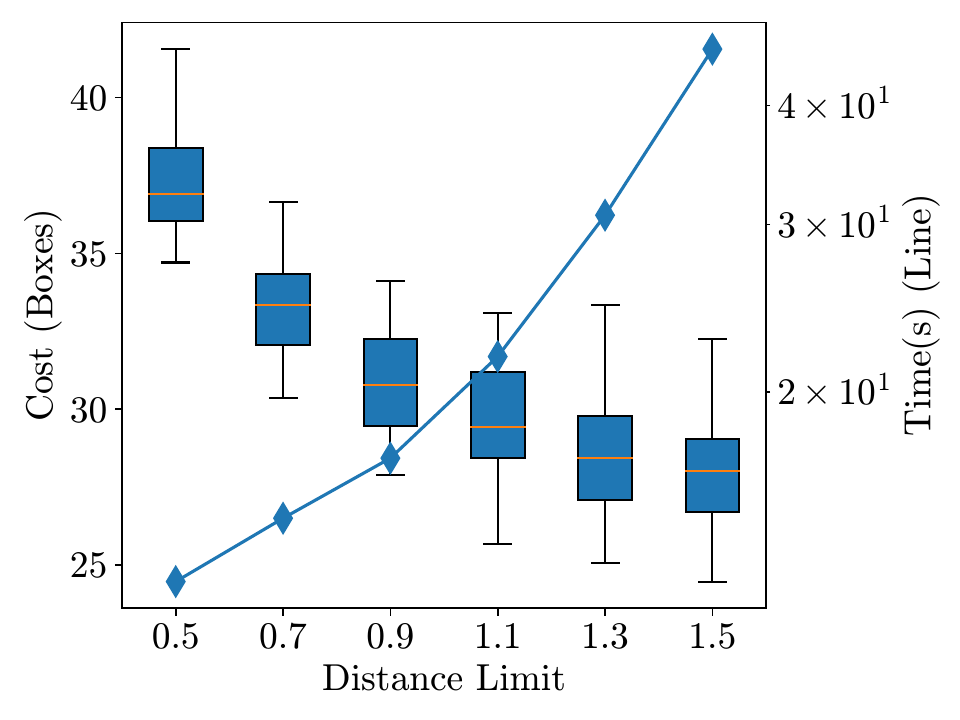}
    \label{fig: limit}
    }
    \caption{The influence of drone speed and distance limit. Time in Fig. (b) is plotted on a logarithm scale.} 
\end{figure}

\subsection{Sensitivity analysis}\label{sec: factor}
In this part, we do the sensitivity analysis for our algorithm on several important parameters for drones, including the drone speed, the drone flying distance limit, and the number of drones. 

\textbf{Drone speed.} We vary the drone speed $s_{\text{dr}}$ from 10 km/h to 110 km/h on the Manhattan map, using 100 instances with five depots and 100 customers. The results, shown in \autoref{fig: speed}, show that the cost almost follows a decrease proportional to $1/s_{\text{dr}}$. This indicates that even when drones are not significantly faster than trucks, the truck group still benefits from equipping drones due to their ability to fly freely without road network restrictions. We also observe a peak in running time as drone speed increases, suggesting that when drone and truck speeds are similar, more pairs of vertices become comparable, making it harder to prune options when solving the MILP. 

\textbf{Distance limit for drones.} We vary the distance limit $r$ from 0.5 km to 1.5 km using the same dataset as in the drone speed experiments. Results shown in \autoref{fig: limit}indicate that the cost decreases sub-linearly as $r$ increases, which is attributed to the limited number of drones and their speed. The running time increases exponentially because the number of vertices on the boundary grows almost linearly with the limit $r$.

\textbf{Number of drones.} We test different numbers of drones $k\in\{0, 1,2,3,4,5\}$ with varying numbers of customers (from $|\mathcal{C}|=50$ to $|\mathcal{C}|=150$), using 100 instances for each scenario with a fixed number of depots $|\mathcal{P}|=5$. Results are presented in \autoref{tab: k}. As the number of drones increases, the route decreases. However, the percentage of cost decrease for each additional drone also diminishes, with faster diminishing observed for fewer customers. This is because the benefit of increasing the number of drones has an upper bound. To benefit from adding one more drone, there should be at least one customer remaining after dispatching all drones and before collecting any of them back, which requires an increasingly larger customer density. Hence, when building a truck group, the number of drones to equip should be carefully evaluated based on customer distribution.
\begin{table}[ht]
  \caption{Cost for different numbers of drones per truck on datasets sampling vertices from uniform distribution and Clusters. $\Delta$ evaluates the cost decrease of adding one drone.}
  \label{tab: k}
  \centering
  \begin{tabular}{cccccccccccc}
    \toprule 
    & $k=0$ & \multicolumn{2}{c}{$k=1$} & \multicolumn{2}{c}{$k=2$} & \multicolumn{2}{c}{$k=3$} & \multicolumn{2}{c}{$k=4$} & \multicolumn{2}{c}{$k=5$}\\
    \cmidrule(r){2-2} \cmidrule(r){3-4} \cmidrule(r){5-6} \cmidrule(r){7-8} \cmidrule(r){9-10} \cmidrule(r){11-12}
    $|\mathcal{C}|$  & Cost & Cost & $\Delta$ & Cost & $\Delta$ & Cost & $\Delta$ & Cost & $\Delta$ & Cost & $\Delta$ \\
    \midrule
    50 & 39.64 & 25.02 & 36.88\% & 22.50 & 10.07\% & 21.69 & 3.60\% & 21.30 & 1.80\% & 21.17 & 0.61\%\\
    70 & 43.57 & 27.04 & 37.94\% & 23.80 & 11.98\% & 22.29 & 6.34\% & 21.48 & 3.63\% & 21.02 & 2.14\%\\
    90 & 49.02 & 30.25 &38.29\% & 25.99 & 14.08\% & 24.05 & 7.46\% & 23.07 & 4.07\% & 22.64 & 1.86\%\\
    110 & 52.83 & 32.84 & 37.83\% & 28.25 & 13.98\% & 25.98 & 8.04\% & 24.82 & 4.46\% & 24.07 & 3.02\%\\
    130 & 56.38 & 35.90 & 36.32\% & 30.80 & 14.21\% & 28.35 & 7.95\% & 26.77 & 5.57\% & 25.88 & 3.32\%\\
    150 & 60.95 & 38.37 & 37.05\% & 32.92 & 14.20\% & 30.01 & 8.84\% & 28.33 & 5.60\% & 27.24 & 3.85\%\\
    \bottomrule
  \end{tabular}
\end{table}

This analysis provides insights into how these key parameters affect the performance and efficiency of our algorithm, helping to guide optimal configuration in various scenarios.

\section{Conclusion} 
\label{sec:conclusion}
In this paper, we have explored a new variant of the Flying Sidekick Traveling Salesman Problem (FSTSP) called MA-FSTSP. This variant is designed to address the complexities of path-finding over practical road networks and capitalize on drones' advantages over ground vehicles in real-world applications. We have proposed a mixed integer linear programming (MILP) model for this problem.

Recognizing that finding the optimal solution for this problem is not scalable, we have developed a novel 3-phase method. Extensive experiments demonstrate that this method efficiently produces high-quality solutions compared to baselines and can solve large-scale instances.

For future work, we identify two main directions. First, to enhance scalability, we propose extending heuristic algorithms like the Lin-Kernighan Heuristic (LKH) \cite{lin1973effective} to solve the Set-TSP in the second phase, replacing the current MILP approach. Second, to better characterize drone flexibility, we suggest expanding possible take-off and landing locations from vertices to arbitrary points along edges. This would transform the problem from discrete to continuous, allowing for more accurate routing of truck-drone delivery systems.

\appendix
\section{Time complexity analysis}
\label{app: tc}
Here, we state the complexity of our 3-phase method. Recall that $k$ is the number of drones, $m$ is the number of depots, $n$ is the number of customers, $\theta$ is the set radius we choose, and $\mathcal{S}_c(\theta)$ is the set of vertices we collected for customer $c$ in its neighborhood. Let $\mathcal B(\mathcal S_c(\theta)$ be the boundary of the set $\mathcal S_c(\theta)$.
\begin{lemma}\label{lem: num}
    $\mathcal{S}_c(\theta)$ contains $O(\theta^2)$ vertices in the set and $O(\theta)$ vertices on the boundary. 
\end{lemma}
\begin{proof}
    Suppose there is a lower bound $l$ for the distance between vertices on the graph; then each vertex has at most 6 neighbors within distance $l$. So the number of vertices we collect to form sets $|\mathcal{S}_\theta|\le 7\frac{\pi\theta^2}{\pi l^2}=O(\theta^2)$ grows in quadratic to the radius $\theta$. And the number of vertices on the boundary $|\mathcal{B}(\mathcal{S}_c(\theta))|\le 2\pi\theta/l=O(\theta)$.
\end{proof}
\begin{theorem}
    In phase 1, Set-NN has time complexity $O(mn\theta^2)$, Set-MST has time complexity $O(m^2+n^2(\theta^4+1))$.
\end{theorem}
\begin{proof}
    In Set-NN, each customer-depot pair needs the computation of set distance, which in total is $O(mn\cdot (1\cdot |\mathcal{S(\theta)})|)=O(mn\theta^2)$ by \autoref{lem: num}. In Set-MST, $O(n^2\cdot |\mathcal{S}(\theta)|^2+mn\cdot|\mathcal{S}(\theta)|+m^2\cdot 1)=O(n^2\theta^4+m^2)$ computation time is needed to form the fully connected graph of customers and depots. Then, the minimum spanning tree can be obtained in $O((m+n)^2+(m+n)\log (m+n))=O(m^2+n^2)$. Finally, partitioning takes $O(m+n)$ by dynamic programming. The overall time complexity is $O(m^2 + n^2(\theta^4+1))$.
\end{proof}
\begin{theorem}
    Without acceleration strategies, the MILP for Set-TSP has $O(n^2(\theta^4+1))$ binary variables, $O(n^2)$ continuous variables, and $O(n^2 + n\theta^2)$ constraints. With acceleration strategies, the MILP reduces the number of binary variables to $O(n^2(\theta^2+1))$ and the number of constraints to $O(n^2+n\theta)$.
\end{theorem}
\begin{proof}
    Without acceleration, the MILP for Set-TSP has $\sum_{p\in\mathcal{P}}\binom{|\mathcal{C}_p|+1}{2}=O(n^2)$ binary variables $\beta$,  $\sum_{c\in\mathcal{C}}|\mathcal{S}_c(\theta)|^2=O(n\theta^4)$ binary variables $\gamma$, and $\sum_{p\in\mathcal{P}}\sum_{c,c'\in\mathcal{C}_p}|\mathcal{S}_c(\theta)|\cdot|\mathcal{S}_{c'}(\theta)|=O(n^2\theta^4)$ binary variables $\delta$. So it has $O(n^2(\theta^4+1))$ binary variables in total. With acceleration strategies, each $\mathcal{S}_c(\theta)$ is replaced by its boundary $\mathcal{B}(\mathcal{S}_c(\theta))$, which gives $O(n^2(\theta^2+1))$ binary variables in total. The MILP also contains $O(n^2)$ continuous variables, the same as the binary variables $\beta$. For the constraints, the MILP has $O(\sum_{p\in\mathcal{P}_i}\sum_{c\in\mathcal{C}_p}(1 + \sum_{c'\in\mathcal{C}_p}1+\sum_{v\in\mathcal{S}_c(\theta)}1))=O(n^2+n\theta^2)$ constraints when not accelerated, and $O(\sum_{p\in\mathcal{P}_i}\sum_{c\in\mathcal{C}_p}(1 + \sum_{c'\in\mathcal{C}_p}1+\sum_{v\in\mathcal B(\mathcal{S}_c(\theta))}1))=O(n^2+n\theta)$ when accelerated. 
\end{proof}
\begin{theorem}\label{thm: complexity}
    The dynamic programming in phase 3 has time complexity $O(knr^6)$.
\end{theorem}
\begin{proof}
    For each truck group route, given the visiting order of customers $\mathcal{O}$, $O(\sum_{s\le |\mathcal{O}|}\sum_{t\le k}\sum_{u\in\mathcal{S}_{\mathcal{O}^s}(r)}\sum_{v\in\mathcal{S}_{\mathcal{O}^{s+t-1}}(r)}1)=O(k|\mathcal{O}|r^4)$
    values of $\textsc{Time}$ needed to be computed, each maximizing over $O(\sum_{u\in\mathcal{S}_{\mathcal{O}^{s+t-2}}(r)})=O(r^2)$ variables. We also have to compute $O(\sum_{s\le |\mathcal{O}|}\sum_{v\in\cup_{c\in\mathcal{O}}\mathcal{S}_c(r)}1)=O(|\mathcal{O}|r^2)$ values of $\textsc{Value}$, each maximizing $O(\sum_{t\le k}\sum_{u\in\mathcal{S}_{\mathcal O^{s-t}}(r)}\sum_{w\in\mathcal{S}_{\mathcal{O}^{s}}(r)}1)=O(kr^4)$ terms. So the time complexity of the dynamic programming is $O(\sum_{p\in\mathcal{P}}k|\mathcal{O}_p|r^4\cdot r^2+|\mathcal{O}_p|r^2\cdot kr^4)=O(knr^6)$.
\end{proof}

\section{Computation of Lower Bounds}
\label{app: lb}
The lower bound for the optimal solution cost of MA-FSTSP is the optimal solution cost given $k=+\infty$ and $s_\text{dr}=+\infty$, which is equivalent to a Set MATSP:
\begin{problem}[\textbf{Set-MATSP}]
    Given a strongly connected directed graph $\mathcal{G}=(\mathcal{V}, \mathcal{E})$, and $m$ trucks starting from different depots in set $\mathcal{D}\subseteq\mathcal{V}$ to visit a set of $n$ customers $\mathcal{C}\subseteq\mathcal{V}$, where a visit to customer $c\in\mathcal{C}$ is defined
    as reaching a vertex $v\in S_c(\frac{r}{2})$, find a set of $m$ truck routes, one for each truck, to visit every customer $c\in\mathcal{C}$ at least once that minimizes the total cost. The cost of a truck route is the total weight of the edges it passes.
\end{problem}

\section*{Declaration of generative AI and AI-assisted technologies in the writing process}

During the preparation of this work the author(s) used Claude in order to improve language. After using this tool/service, the author(s) reviewed and edited the content as needed and take(s) full responsibility for the content of the publication.

\bibliographystyle{elsarticle-num-names} 
\bibliography{main}

\begin{thebibliography}{33}
\expandafter\ifx\csname natexlab\endcsname\relax\def\natexlab#1{#1}\fi
\providecommand{\url}[1]{\texttt{#1}}
\providecommand{\href}[2]{#2}
\providecommand{\path}[1]{#1}
\providecommand{\DOIprefix}{doi:}
\providecommand{\ArXivprefix}{arXiv:}
\providecommand{\URLprefix}{URL: }
\providecommand{\Pubmedprefix}{pmid:}
\providecommand{\doi}[1]{\href{http://dx.doi.org/#1}{\path{#1}}}
\providecommand{\Pubmed}[1]{\href{pmid:#1}{\path{#1}}}
\providecommand{\bibinfo}[2]{#2}
\ifx\xfnm\relax \def\xfnm[#1]{\unskip,\space#1}\fi
\bibitem[{Murray and Chu(2015)}]{murray2015flying}
\bibinfo{author}{C.~C. Murray}, \bibinfo{author}{A.~G. Chu},
\newblock \bibinfo{title}{\href{https://www.sciencedirect.com/science/article/pii/S0968090X15000844}{The flying sidekick traveling salesman problem: Optimization of drone-assisted parcel delivery}},
\newblock \bibinfo{journal}{Transportation Research Part C: Emerging Technologies} \bibinfo{volume}{54} (\bibinfo{year}{2015}) \bibinfo{pages}{86--109}.
\bibitem[{Lin et~al.(2022)Lin, Chen, Han, Chen et~al.}]{lin2022discrete}
\bibinfo{author}{M.~Lin}, \bibinfo{author}{Y.~Chen}, \bibinfo{author}{R.~Han}, \bibinfo{author}{Y.~Chen}, et~al.,
\newblock \bibinfo{title}{\href{https://www.hindawi.com/journals/ddns/2022/1811288/}{Discrete optimization on truck-drone collaborative transportation system for delivering medical resources}},
\newblock \bibinfo{journal}{Discrete Dynamics in Nature and Society} \bibinfo{volume}{2022} (\bibinfo{year}{2022}).
\bibitem[{Gao et~al.(2023)Gao, Zhen, Laporte, and He}]{gao2023scheduling}
\bibinfo{author}{J.~Gao}, \bibinfo{author}{L.~Zhen}, \bibinfo{author}{G.~Laporte}, \bibinfo{author}{X.~He},
\newblock \bibinfo{title}{\href{https://www.sciencedirect.com/science/article/abs/pii/S1366554523002557?casa_token=iR4GjN8tmjUAAAAA:x23KLB7PrRurhxXLZ_UxbYPwBy3RDoEP71K8-9VAGnaibI_Gp2EiLPHNOYm_G9jErLxb1boxnQ}{Scheduling trucks and drones for cooperative deliveries}},
\newblock \bibinfo{journal}{Transportation Research Part E: Logistics and Transportation Review} \bibinfo{volume}{178} (\bibinfo{year}{2023}) \bibinfo{pages}{103267}.
\bibitem[{Choudhury et~al.(2021)Choudhury, Solovey, Kochenderfer, and Pavone}]{choudhury2021efficient}
\bibinfo{author}{S.~Choudhury}, \bibinfo{author}{K.~Solovey}, \bibinfo{author}{M.~J. Kochenderfer}, \bibinfo{author}{M.~Pavone},
\newblock \bibinfo{title}{\href{https://www.jair.org/index.php/jair/article/view/12450}{Efficient large-scale multi-drone delivery using transit networks}},
\newblock \bibinfo{journal}{Journal of Artificial Intelligence Research} \bibinfo{volume}{70} (\bibinfo{year}{2021}) \bibinfo{pages}{757--788}.
\bibitem[{Greshler et~al.(2021)Greshler, Gordon, Salzman, and Shimkin}]{greshler2021cooperative}
\bibinfo{author}{N.~Greshler}, \bibinfo{author}{O.~Gordon}, \bibinfo{author}{O.~Salzman}, \bibinfo{author}{N.~Shimkin},
\newblock \bibinfo{title}{\href{https://ieeexplore.ieee.org/stamp/stamp.jsp?arnumber=9620590}{Cooperative multi-agent path finding: Beyond path planning and collision avoidance}},
\newblock in: \bibinfo{booktitle}{2021 International Symposium on Multi-Robot and Multi-Agent Systems (MRS)}, \bibinfo{organization}{IEEE}, \bibinfo{year}{2021}, pp. \bibinfo{pages}{20--28}.
\bibitem[{Choudhury et~al.(2021)Choudhury, Solovey, Kochenderfer, and Pavone}]{choudhury2021coordinated}
\bibinfo{author}{S.~Choudhury}, \bibinfo{author}{K.~Solovey}, \bibinfo{author}{M.~Kochenderfer}, \bibinfo{author}{M.~Pavone},
\newblock \bibinfo{title}{\href{https://arxiv.org/abs/2110.08802}{Coordinated multi-agent pathfinding for drones and trucks over road networks}},
\newblock \bibinfo{journal}{arXiv preprint arXiv:2110.08802}  (\bibinfo{year}{2021}).
\bibitem[{Marcucci et~al.(2021)Marcucci, Umenberger, Parrilo, and Tedrake}]{marcucci2021shortest}
\bibinfo{author}{T.~Marcucci}, \bibinfo{author}{J.~Umenberger}, \bibinfo{author}{P.~A. Parrilo}, \bibinfo{author}{R.~Tedrake},
\newblock \bibinfo{title}{\href{https://arxiv.org/abs/2101.11565}{Shortest paths in graphs of convex sets}},
\newblock \bibinfo{journal}{arXiv preprint arXiv:2101.11565}  (\bibinfo{year}{2021}).
\bibitem[{Mbiadou~Saleu et~al.(2018)Mbiadou~Saleu, Deroussi, Feillet, Grangeon, and Quilliot}]{mbiadou2018iterative}
\bibinfo{author}{R.~G. Mbiadou~Saleu}, \bibinfo{author}{L.~Deroussi}, \bibinfo{author}{D.~Feillet}, \bibinfo{author}{N.~Grangeon}, \bibinfo{author}{A.~Quilliot},
\newblock \bibinfo{title}{\href{https://onlinelibrary.wiley.com/doi/abs/10.1002/net.21846}{An iterative two-step heuristic for the parallel drone scheduling traveling salesman problem}},
\newblock \bibinfo{journal}{Networks} \bibinfo{volume}{72} (\bibinfo{year}{2018}) \bibinfo{pages}{459--474}.
\bibitem[{Karak and Abdelghany(2019)}]{karak2019hybrid}
\bibinfo{author}{A.~Karak}, \bibinfo{author}{K.~Abdelghany},
\newblock \bibinfo{title}{\href{https://www.sciencedirect.com/science/article/pii/S0968090X18312932}{The hybrid vehicle-drone routing problem for pick-up and delivery services}},
\newblock \bibinfo{journal}{Transportation Research Part C: Emerging Technologies} \bibinfo{volume}{102} (\bibinfo{year}{2019}) \bibinfo{pages}{427--449}.
\bibitem[{Murray and Raj(2020)}]{murray2020multiple}
\bibinfo{author}{C.~C. Murray}, \bibinfo{author}{R.~Raj},
\newblock \bibinfo{title}{\href{https://www.sciencedirect.com/science/article/abs/pii/S0968090X19302505}{The multiple flying sidekicks traveling salesman problem: Parcel delivery with multiple drones}},
\newblock \bibinfo{journal}{Transportation Research Part C: Emerging Technologies} \bibinfo{volume}{110} (\bibinfo{year}{2020}) \bibinfo{pages}{368--398}.
\bibitem[{Cavani et~al.(2021)Cavani, Iori, and Roberti}]{cavani2021exact}
\bibinfo{author}{S.~Cavani}, \bibinfo{author}{M.~Iori}, \bibinfo{author}{R.~Roberti},
\newblock \bibinfo{title}{\href{https://www.sciencedirect.com/science/article/pii/S0968090X21002928}{Exact methods for the traveling salesman problem with multiple drones}},
\newblock \bibinfo{journal}{Transportation Research Part C: Emerging Technologies} \bibinfo{volume}{130} (\bibinfo{year}{2021}) \bibinfo{pages}{103280}.
\bibitem[{Salama and Srinivas(2022)}]{salama2022collaborative}
\bibinfo{author}{M.~R. Salama}, \bibinfo{author}{S.~Srinivas},
\newblock \bibinfo{title}{\href{https://www.sciencedirect.com/science/article/abs/pii/S1366554522001776}{Collaborative truck multi-drone routing and scheduling problem: Package delivery with flexible launch and recovery sites}},
\newblock \bibinfo{journal}{Transportation Research Part E: Logistics and Transportation Review} \bibinfo{volume}{164} (\bibinfo{year}{2022}) \bibinfo{pages}{102788}.
\bibitem[{Bruni et~al.(2022)Bruni, Khodaparasti, and Moshref-Javadi}]{bruni2022logic}
\bibinfo{author}{M.~E. Bruni}, \bibinfo{author}{S.~Khodaparasti}, \bibinfo{author}{M.~Moshref-Javadi},
\newblock \bibinfo{title}{\href{https://www.sciencedirect.com/science/article/abs/pii/S0305054822001228}{A logic-based Benders decomposition method for the multi-trip traveling repairman problem with drones}},
\newblock \bibinfo{journal}{Computers \& Operations Research} \bibinfo{volume}{145} (\bibinfo{year}{2022}) \bibinfo{pages}{105845}.
\bibitem[{Sacramento et~al.(2019)Sacramento, Pisinger, and Ropke}]{sacramento2019adaptive}
\bibinfo{author}{D.~Sacramento}, \bibinfo{author}{D.~Pisinger}, \bibinfo{author}{S.~Ropke},
\newblock \bibinfo{title}{\href{https://www.sciencedirect.com/science/article/pii/S0968090X18303218}{An adaptive large neighborhood search metaheuristic for the vehicle routing problem with drones}},
\newblock \bibinfo{journal}{Transportation Research Part C: Emerging Technologies} \bibinfo{volume}{102} (\bibinfo{year}{2019}) \bibinfo{pages}{289--315}.
\bibitem[{Chiang et~al.(2019)Chiang, Li, Shang, and Urban}]{chiang2019impact}
\bibinfo{author}{W.-C. Chiang}, \bibinfo{author}{Y.~Li}, \bibinfo{author}{J.~Shang}, \bibinfo{author}{T.~L. Urban},
\newblock \bibinfo{title}{\href{https://www.sciencedirect.com/science/article/pii/S0306261919305252}{Impact of drone delivery on sustainability and cost: Realizing the UAV potential through vehicle routing optimization}},
\newblock \bibinfo{journal}{Applied energy} \bibinfo{volume}{242} (\bibinfo{year}{2019}) \bibinfo{pages}{1164--1175}.
\bibitem[{Lu et~al.(2022)Lu, Yang, and Yang}]{lu2022multi}
\bibinfo{author}{Y.~Lu}, \bibinfo{author}{C.~Yang}, \bibinfo{author}{J.~Yang},
\newblock \bibinfo{title}{\href{https://link.springer.com/article/10.1007/s10479-022-04816-y}{A multi-objective humanitarian pickup and delivery vehicle routing problem with drones}},
\newblock \bibinfo{journal}{Annals of Operations Research} \bibinfo{volume}{319} (\bibinfo{year}{2022}) \bibinfo{pages}{291--353}.
\bibitem[{Luo et~al.(2022)Luo, Gu, Poon, Liu, and Lim}]{luo2022last}
\bibinfo{author}{Z.~Luo}, \bibinfo{author}{R.~Gu}, \bibinfo{author}{M.~Poon}, \bibinfo{author}{Z.~Liu}, \bibinfo{author}{A.~Lim},
\newblock \bibinfo{title}{\href{https://www.sciencedirect.com/science/article/abs/pii/S0305054822002453}{A last-mile drone-assisted one-to-one pickup and delivery problem with multi-visit drone trips}},
\newblock \bibinfo{journal}{Computers \& Operations Research} \bibinfo{volume}{148} (\bibinfo{year}{2022}) \bibinfo{pages}{106015}.
\bibitem[{Kitjacharoenchai et~al.(2019)Kitjacharoenchai, Ventresca, Moshref-Javadi, Lee, Tanchoco, and Brunese}]{kitjacharoenchai2019multiple}
\bibinfo{author}{P.~Kitjacharoenchai}, \bibinfo{author}{M.~Ventresca}, \bibinfo{author}{M.~Moshref-Javadi}, \bibinfo{author}{S.~Lee}, \bibinfo{author}{J.~M. Tanchoco}, \bibinfo{author}{P.~A. Brunese},
\newblock \bibinfo{title}{\href{https://www.sciencedirect.com/science/article/pii/S0360835219300245}{Multiple traveling salesman problem with drones: Mathematical model and heuristic approach}},
\newblock \bibinfo{journal}{Computers \& Industrial Engineering} \bibinfo{volume}{129} (\bibinfo{year}{2019}) \bibinfo{pages}{14--30}.
\bibitem[{Tamke and Buscher(2021)}]{tamke2021branch}
\bibinfo{author}{F.~Tamke}, \bibinfo{author}{U.~Buscher},
\newblock \bibinfo{title}{\href{https://www.sciencedirect.com/science/article/abs/pii/S0191261520304410}{A branch-and-cut algorithm for the vehicle routing problem with drones}},
\newblock \bibinfo{journal}{Transportation Research Part B: Methodological} \bibinfo{volume}{144} (\bibinfo{year}{2021}) \bibinfo{pages}{174--203}.
\bibitem[{Chen et~al.(2021)Chen, Demir, and Huang}]{chen2021adaptive}
\bibinfo{author}{C.~Chen}, \bibinfo{author}{E.~Demir}, \bibinfo{author}{Y.~Huang},
\newblock \bibinfo{title}{\href{https://www.sciencedirect.com/science/article/abs/pii/S037722172100120X}{An adaptive large neighborhood search heuristic for the vehicle routing problem with time windows and delivery robots}},
\newblock \bibinfo{journal}{European journal of operational research} \bibinfo{volume}{294} (\bibinfo{year}{2021}) \bibinfo{pages}{1164--1180}.
\bibitem[{Gao et~al.(2023)Gao, Zhen, and Wang}]{gao2023multi}
\bibinfo{author}{J.~Gao}, \bibinfo{author}{L.~Zhen}, \bibinfo{author}{S.~Wang},
\newblock \bibinfo{title}{\href{https://www.sciencedirect.com/science/article/abs/pii/S0968090X23003972}{Multi-trucks-and-drones cooperative pickup and delivery problem}},
\newblock \bibinfo{journal}{Transportation Research Part C: Emerging Technologies} \bibinfo{volume}{157} (\bibinfo{year}{2023}) \bibinfo{pages}{104407}.
\bibitem[{Carlsson and Song(2018)}]{carlsson2018coordinated}
\bibinfo{author}{J.~G. Carlsson}, \bibinfo{author}{S.~Song},
\newblock \bibinfo{title}{\href{https://pubsonline.informs.org/doi/10.1287/mnsc.2017.2824}{Coordinated logistics with a truck and a drone}},
\newblock \bibinfo{journal}{Management Science} \bibinfo{volume}{64} (\bibinfo{year}{2018}) \bibinfo{pages}{4052--4069}.
\bibitem[{Li et~al.(2022)Li, Chen, Wang, and Zhao}]{li2022truck}
\bibinfo{author}{H.~Li}, \bibinfo{author}{J.~Chen}, \bibinfo{author}{F.~Wang}, \bibinfo{author}{Y.~Zhao},
\newblock \bibinfo{title}{\href{https://onlinelibrary.wiley.com/doi/10.1002/nav.22053}{Truck and drone routing problem with synchronization on arcs}},
\newblock \bibinfo{journal}{Naval Research Logistics (NRL)} \bibinfo{volume}{69} (\bibinfo{year}{2022}) \bibinfo{pages}{884--901}.
\bibitem[{Ho et~al.(2008)Ho, Ho, Ji, and Lau}]{ho2008hybrid}
\bibinfo{author}{W.~Ho}, \bibinfo{author}{G.~T. Ho}, \bibinfo{author}{P.~Ji}, \bibinfo{author}{H.~C. Lau},
\newblock \bibinfo{title}{\href{https://www.sciencedirect.com/science/article/pii/S0377221705006983}{A hybrid genetic algorithm for the multi-depot vehicle routing problem}},
\newblock \bibinfo{journal}{Engineering applications of artificial intelligence} \bibinfo{volume}{21} (\bibinfo{year}{2008}) \bibinfo{pages}{548--557}.
\bibitem[{Salhi et~al.(2014)Salhi, Imran, and Wassan}]{salhi2014multi}
\bibinfo{author}{S.~Salhi}, \bibinfo{author}{A.~Imran}, \bibinfo{author}{N.~A. Wassan},
\newblock \bibinfo{title}{\href{https://www.sciencedirect.com/science/article/pii/S0305054813001408}{The multi-depot vehicle routing problem with heterogeneous vehicle fleet: Formulation and a variable neighborhood search implementation}},
\newblock \bibinfo{journal}{Computers \& Operations Research} \bibinfo{volume}{52} (\bibinfo{year}{2014}) \bibinfo{pages}{315--325}.
\bibitem[{Geetha et~al.(2012)Geetha, Vanathi, and Poonthalir}]{geetha2012metaheuristic}
\bibinfo{author}{S.~Geetha}, \bibinfo{author}{P.~Vanathi}, \bibinfo{author}{G.~Poonthalir},
\newblock \bibinfo{title}{\href{https://www.tandfonline.com/doi/abs/10.1080/08839514.2012.727344}{Metaheuristic approach for the multi-depot vehicle routing problem}},
\newblock \bibinfo{journal}{Applied Artificial Intelligence} \bibinfo{volume}{26} (\bibinfo{year}{2012}) \bibinfo{pages}{878--901}.
\bibitem[{Yang and Fan(2024)}]{yang2024hierarchical}
\bibinfo{author}{R.~Yang}, \bibinfo{author}{C.~Fan},
\newblock \bibinfo{title}{\href{https://ieeexplore.ieee.org/abstract/document/10500884}{A Hierarchical Framework for Solving the Constrained Multiple Depot Traveling Salesman Problem}},
\newblock \bibinfo{journal}{IEEE Robotics and Automation Letters}  (\bibinfo{year}{2024}).
\bibitem[{Gavish and Graves(1978)}]{gavish1978travelling}
\bibinfo{author}{B.~Gavish}, \bibinfo{author}{S.~C. Graves},
\newblock \bibinfo{title}{\href{https://dspace.mit.edu/handle/1721.1/5363}{The travelling salesman problem and related problems}}  (\bibinfo{year}{1978}).
\bibitem[{{Gurobi Optimization, LLC}(2023)}]{gurobi}
\bibinfo{author}{{Gurobi Optimization, LLC}}, \bibinfo{title}{{Gurobi Optimizer Reference Manual}}, \bibinfo{year}{2023}. \URLprefix \url{https://www.gurobi.com}.
\bibitem[{Chinnasamy et~al.(2022)Chinnasamy, Ramachandran, Amudha, and Ramu}]{chinnasamy2022review}
\bibinfo{author}{S.~Chinnasamy}, \bibinfo{author}{M.~Ramachandran}, \bibinfo{author}{M.~Amudha}, \bibinfo{author}{K.~Ramu},
\newblock \bibinfo{title}{\href{https://www.academia.edu/download/87163532/1.-A-Review-on-Hill-Climbing-Optimization-Methodology.pdf}{A review on hill climbing optimization methodology}},
\newblock \bibinfo{journal}{Recent Trends in Management and Commerce} \bibinfo{volume}{3} (\bibinfo{year}{2022}).
\bibitem[{Blahoudek et~al.(2020)Blahoudek, Br{\'a}zdil, Novotn{\`y}, Ornik, Thangeda, and Topcu}]{blahoudek2020qualitative}
\bibinfo{author}{F.~Blahoudek}, \bibinfo{author}{T.~Br{\'a}zdil}, \bibinfo{author}{P.~Novotn{\`y}}, \bibinfo{author}{M.~Ornik}, \bibinfo{author}{P.~Thangeda}, \bibinfo{author}{U.~Topcu},
\newblock \bibinfo{title}{\href{https://link.springer.com/chapter/10.1007/978-3-030-53291-8_22\#Sec11}{Qualitative controller synthesis for consumption Markov decision processes}},
\newblock in: \bibinfo{booktitle}{International Conference on Computer Aided Verification}, \bibinfo{organization}{Springer}, \bibinfo{year}{2020}, pp. \bibinfo{pages}{421--447}.
\bibitem[{{OpenStreetMap contributors}(2017)}]{OpenStreetMap}
\bibinfo{author}{{OpenStreetMap contributors}}, \bibinfo{title}{{Planet dump retrieved from https://planet.osm.org }}, \bibinfo{howpublished}{\url{ https://www.openstreetmap.org}}, \bibinfo{year}{2017}.
\bibitem[{Lin and Kernighan(1973)}]{lin1973effective}
\bibinfo{author}{S.~Lin}, \bibinfo{author}{B.~W. Kernighan},
\newblock \bibinfo{title}{\href{https://pubsonline.informs.org/doi/abs/10.1287/opre.21.2.498}{An effective heuristic algorithm for the traveling-salesman problem}},
\newblock \bibinfo{journal}{Operations research} \bibinfo{volume}{21} (\bibinfo{year}{1973}) \bibinfo{pages}{498--516}.

\end{thebibliography}

\end{document}